\documentclass[twoside, 11pt]{article}

\usepackage[abbrvbib, preprint]{jmlr2e}

\usepackage{amssymb,amsfonts,amsmath}
\makeatletter
\usepackage{xparse}[2019/10/11]
\makeatother

\usepackage{xcolor}
\usepackage{cleveref}
\Crefname{ALC@unique}{Line}{Lines}

\usepackage{algorithm,algorithmicx,algpseudocode}

\Crefname{ALC@unique}{Line}{Lines}

\definecolor{darkgreen}{RGB}{0,128,0}
\definecolor{darkblue}{RGB}{0,0,128}
\definecolor{darkred}{RGB}{128,0,0}

\DeclareDocumentCommand \prArg{mm}
{(
\IfNoValueTF{#2}{#1}{#1 \mid #2}
)}

\DeclareDocumentCommand \newProbabilityFormat{r<>m}
{
	\DeclareDocumentCommand #1 {e{_}e{^}>{\SplitArgument{1}{|}}r()}
	{
		\IfNoValueTF{##1}
		{
			\IfNoValueTF{##2}
			{#2\prArg##3}
			{#2^{##2}\prArg##3}
		}
		{
			\IfNoValueTF{##2}
			{#2_{##1}\prArg##3}
			{#2_{##1}^{##2}\prArg##3}
		}
	}
}

\DeclareDocumentCommand \fFunction {m} {{#1}}
\DeclareDocumentCommand \fSet {m} {\mathcal{#1}}

\DeclareDocumentCommand \fLimOperator {m} {\mathop{\vphantom{\lim}\mathchoice {\hbox{#1}} {\vcenter{\hbox{#1}}}{#1}{#1}}\displaylimits}

\DeclareDocumentCommand \newScalar{r<>m}
{
	\DeclareDocumentCommand #1 {} {{#2}}
}

\DeclareDocumentCommand \newVector{r<>m}
{
	\DeclareDocumentCommand #1 {} {\boldsymbol{#2}}
}

\DeclareDocumentCommand \newMatrix{r<>m}
{
	\DeclareDocumentCommand #1 {} {\boldsymbol{#2}}
}

\DeclareDocumentCommand \newProbability{r<>m}
{
	\newProbabilityFormat<#1>{#2}
}

\DeclareDocumentCommand \newFunction{r<>m}
{
	\DeclareDocumentCommand #1 {} {\fFunction{#2}}
}

\DeclareDocumentCommand \newSet{r<>m}
{
	\DeclareDocumentCommand #1 {} {\fSet{#2}}
}

\DeclareDocumentCommand \argmin {} {\fLimOperator{argmin}}
\DeclareDocumentCommand \argmax {} {\fLimOperator{argmax}}

\makeatletter
\newcommand{\thickhline}{%
    \noalign {\ifnum 0=`}\fi \hrule height 1pt
    \futurelet \reserved@a \@xhline
}
\makeatother

\DeclareMathOperator{\expect}{\mathbb{E}}

\DeclareMathOperator{\erf}{erf}

\DeclareDocumentCommand \tAdd { m } {#1}
\DeclareDocumentCommand \uAdd { m } {#1}

\newVector<\vA>{a}
\newVector<\vB>{b}
\newVector<\vC>{c}
\newVector<\vD>{d}
\newVector<\vE>{e}
\newVector<\vF>{f}
\newVector<\vG>{g}
\newVector<\vH>{h}
\newVector<\vI>{i}
\newVector<\vJ>{j}
\newVector<\vK>{k}
\newVector<\vL>{l}
\newVector<\vM>{m}
\newVector<\vN>{n}
\newVector<\vP>{p}
\newVector<\vQ>{q}
\newVector<\vR>{r}
\newVector<\vS>{s}
\newVector<\vT>{t}
\newVector<\vU>{u}
\newVector<\vV>{v}
\newVector<\vW>{w}
\newVector<\vX>{x}
\newVector<\vY>{y}
\newVector<\vZ>{z}
\newVector<\vAlpha>{\alpha}
\newVector<\vBeta>{\beta}
\newVector<\vGamma>{\gamma}
\newVector<\vDelta>{\delta}
\newVector<\vEpsilon>{\varepsilon}
\newVector<\vZeta>{\zeta}
\newVector<\vEta>{\eta}
\newVector<\vTheta>{\theta}
\newVector<\vKappa>{\kappa}
\newVector<\vLambda>{\lambda}
\newVector<\vMu>{\mu}
\newVector<\vNu>{\nu}
\newVector<\vXi>{\xi}
\newVector<\vPi>{\pi}
\newVector<\vRho>{\rho}
\newVector<\vSigma>{\sigma}
\newVector<\vTau>{\tau}
\newVector<\vUpsilon>{\upsilon}
\newVector<\vPhi>{\varphi}
\newVector<\vChi>{\chi}
\newVector<\vPsi>{\psi}
\newVector<\vOmega>{\omega}

\newMatrix<\mA>{A}
\newMatrix<\mB>{B}
\newMatrix<\mC>{C}
\newMatrix<\mD>{D}
\newMatrix<\mE>{E}
\newMatrix<\mF>{F}
\newMatrix<\mG>{G}
\newMatrix<\mH>{H}
\newMatrix<\mI>{I}
\newMatrix<\mJ>{J}
\newMatrix<\mK>{K}
\newMatrix<\mL>{L}
\newMatrix<\mM>{M}
\newMatrix<\mN>{N}
\newMatrix<\mP>{P}
\newMatrix<\mQ>{Q}
\newMatrix<\mR>{R}
\newMatrix<\mS>{S}
\newMatrix<\mT>{T}
\newMatrix<\mU>{U}
\newMatrix<\mV>{V}
\newMatrix<\mW>{W}
\newMatrix<\mX>{X}
\newMatrix<\mY>{Y}
\newMatrix<\mZ>{Z}
\newMatrix<\mGamma>{\Gamma}
\newMatrix<\mDelta>{\Delta}
\newMatrix<\mTheta>{\Theta}
\newMatrix<\mLambda>{\Lambda}
\newMatrix<\mXi>{\Xi}
\newMatrix<\mPi>{\Pi}
\newMatrix<\mSigma>{\Sigma}
\newMatrix<\mUpsilon>{\Upsilon}
\newMatrix<\mPhi>{\Phi}
\newMatrix<\mPsi>{\Psi}
\newMatrix<\mOmega>{\Omega}

\DeclareDocumentCommand \thmTitle{m} {\textbf{#1}}
\DeclareDocumentCommand \integ{m}{\!\int\!#1\,}

\newScalar<\nData>{n}
\newScalar<\ks>{k^*\!}

\DeclareDocumentCommand \eps{} {\varepsilon}

\DeclareDocumentCommand \entropy{m} {S\!\left[\,#1\,\right]}
\DeclareDocumentCommand \KL{mm} {D\!\left[\,#1\,\middle\|\,#2\,\right]}

\DeclareDocumentCommand \info{mmm} {\mathbb{I}_{#1}\!\left[\,#2\,\middle\|\,#3\,\right]}

\newFunction<\fInterpret>{\varphi}
\newFunction<\fIntr>{\Omega}
\newFunction<\fPrior>{\alpha}
\newFunction<\fLabel>{\mathcal{L}}
\newFunction<\fPars>{\mathcal{P}}
\newFunction<\fKol>{\mathcal{K}}
\newFunction<\fLen>{\mathcal{\ell}}

\newVector<\vTh>{\theta}
\newVector<\vThr>{\check{\theta}}
\newVector<\vThs>{\theta^*\!}
\newVector<\vThMle>{\bar{\theta}}

\newVector<\vXr>{\check{x}}
\newVector<\vYr>{\check{y}} 
\newVector<\vPsir>{\check{\psi}}
\newVector<\vYu>{\hat{y}} 
\newVector<\vYf>{\hat{y}} 
\newVector<\vZr>{\check{z}}
\newVector<\vZs>{z^*\!}
\newVector<\vWr>{\check{w}}
\newVector<\vGr>{\check{g}}
\newMatrix<\mUc>{\tilde{U}}
\newVector<\vHc>{\tilde{h}}
\newVector<\vSr>{\check{s}}

\newSet<\sData>{\mathcal{D}}
\newSet<\sFunc>{\mathcal{\Theta}}
\newSet<\sRep>{\mathcal{R}}

\newProbability<\pP>{\mathbf{p}}
\newProbability<\pQ>{\mathbf{q}}
\newProbability<\pQs>{\mathbf{q^*}\!}
\newProbability<\pR>{\mathbf{r}}
\newProbability<\pRs>{\mathbf{r^*}\!}
\newProbability<\pS>{\mathbf{s}}
\newProbability<\pG>{\mathbf{g}}
\newProbability<\pEta>{\boldsymbol{\eta}}
\newProbability<\pDelta>{\boldsymbol{\delta}}
\newProbability<\pN>{\mathcal{N}}
\newProbability<\pGen>{\mathbf{g}}

\DeclareDocumentCommand \submitBlank{}{}

\makeatletter
\def\cleartheorem#1{%
    \expandafter\let\csname#1\endcsname\relax
    \expandafter\let\csname c@#1\endcsname\relax
}
\makeatother

\cleartheorem{corollary}
\newcounter{corollarycount}
\newtheorem{corollary}[corollarycount]{Corollary}
\Crefname{corollarycount}{Corollary}{Corollaries}

\cleartheorem{definition}
\newcounter{definitioncount}
\newtheorem{definition}[definitioncount]{Definition}
\Crefname{definitioncount}{Definition}{Definitions}

\ShortHeadings{Parsimonious Inference}{J. A. Duersch and T. A. Catanach}
\firstpageno{1}

\begin{document}

\title{Parsimonious Inference}

\author{\name Jed A. Duersch \email jaduers@sandia.gov \\
       \name Thomas A. Catanach \email tacatan@sandia.gov \\
       \addr Sandia National Laboratories\\
       Livermore, CA 94550, United States}

\maketitle


\begin{abstract}
Bayesian inference provides a uniquely rigorous approach to obtain principled justification for uncertainty in predictions,
yet it is difficult to articulate suitably general prior belief in the machine learning context,
where computational architectures are pure abstractions subject to frequent modifications by practitioners attempting to improve results.
Parsimonious inference is an information-theoretic formulation of inference over arbitrary architectures that formalizes Occam's Razor; we prefer simple and sufficient explanations.
Our universal hyperprior assigns plausibility to prior descriptions, encoded as sequences of symbols, by expanding on the core relationships between program length, Kolmogorov complexity, and Solomonoff's algorithmic probability.
We then cast learning as information minimization over our composite change in belief when an architecture is specified, training data are observed, and model parameters are inferred.
By distinguishing model complexity from prediction information, our framework also quantifies the phenomenon of memorization. \vspace{2mm}\\
Although our theory is general, it is most critical when datasets are limited, e.g.~small or skewed.
We develop novel algorithms for polynomial regression and random forests that are suitable for such data, as demonstrated by our experiments.
Our approaches combine efficient encodings with prudent sampling strategies to construct predictive ensembles without cross-validation,
thus addressing a fundamental challenge in how to efficiently obtain predictions from data.
\end{abstract}



\section{Introduction}
\label{sec:introduction}

We began this investigation desiring to understand the relationship between prior belief and the resulting uncertainty in predictions obtained from inference
in the hope that new insights would provide a sound basis to improve prediction credibility in machine learning.
The mathematical and epistemological foundations of rational belief, from which the laws of probability and Bayesian inference are derived as an extended logic from binary propositional logic \citep{Cox1946},
lead us to assert the central role of Bayesian inference in obtaining rigorous justification for uncertainty in predictions.
Although this foundation of reason holds generally, it is critical when we need to learn robust predictions from limited datasets.
Yet, applying Bayesian inference within the machine learning context requires addressing a fundamental challenge: inference requires prior belief.
When the amount of evidence contained within a dataset regarding a phenomenon of interest is extremely limited, specifying prior belief is not merely an inconvenience;
it is the dominant source of uncertainty in predictions.
Examples of such data limitations include having few observations, noisy measurements, skewed or highly imbalanced labels of interest, or even a degree of mislabeling in the data.    

When predictive models integrate well-understood physical principles, they are often accompanied by physically plausible parameter ranges that provide a strong basis for prior belief.
Likewise, canonical priors are acceptable for simple approximations with relatively few unconstrained parameters in comparison to the size of the dataset intended for inference.
\citet{Kass1996} give a thorough survey of related work.
In contrast, the machine learning paradigm seeks to instrument arbitrary algorithms with high parameter dimensionality.
A typical architecture may have tens of thousands, or perhaps millions, of free parameters.
In this setting, the sensitivity of predictions to an arbitrary choice of prior belief may be unacceptable for applications of consequence \citep{Owhadi2015}.


\subsection{Our Contributions}

Expanding on the work of \citet{Solomonoff1964a,Solomonoff1964b,Solomonoff2009}, \citet{Kolmogorov1965}, \citet{Rissanen1983,Rissanen1984}, and \citet{Hutter2007},
we develop a theoretical framework that assigns plausibility to arbitrary inference architectures.
Just as Solomonoff derives algorithmic probability from program length, \tAdd{the minimum} number of bits needed to encode a program for a specified \tAdd{Universal Turing Machine (UTM)},
we show how a modest generalization yields a universal hyperprior over symbolic encodings of ordinary priors.
We may regard an ordinary prior, that which is typically used in Bayesian inference, as a restricted state of belief from a general universe of potential explanatory models.
Within our framework, every choice of computational architecture, and associated prior over model parameters, is just a restriction of prior belief.
Our hyperprior provides a means to measure and control the complexity of such choices.

We show how our theory of information \citep{Duersch2020}, \Cref{thm:info}, allows us to derive a training objective from the information that is created when we select a prior representation, observe the training data,
and either infer the posterior distribution or construct a variational approximation of it.
\citet{Zhang2018} provide a thorough survey \tAdd{of recent work on variational inference}.
Our main result, \Cref{thm:parsimony_optimization},
clarifies how we may understand learning as an information optimization problem.
Our parsimony objective separates into three components:
\begin{itemize}
\item Encoding information contained within a symbolic description of prior belief;
\item Model information gained through inference using evidence;
\item Predictive information gained regarding the observed labels from plausible models.
\end{itemize}
In our derivation, the first two terms appear with negative signs and the third with a positive sign, revealing how our theory suppresses complexity as an intrinsic tradeoff against increased agreement with observed labels.
The first component guards against excessive complexity in our description of prior belief and the second guards against priors that are poorly suited to our data.
In contrast, the third component promotes agreement between resulting predictions and the data.
\tAdd{The main distinction between the second and third components is that model information is measured in the space of explanations, whereas predictive information measures the result of applying plausible models to our data.}

\uAdd{We review work on universal priors over integers, corresponding binary representations, and show how a simple integer encoding approximates the scaling invariance of Jeffreys prior.} 
We then demonstrate this theory with two learning prototypes.

Our first algorithm casts polynomial regression within this framework, predicting a distribution over continuous outcomes from a continuous input.
By setting the maximum polynomial degree to be much higher than the data merits, standard machine learning training strategies are susceptible to memorization,
as we demonstrate by applying gradient based training with leave-one-out cross-validation.
In contrast, our prototype discovers much simpler models from the same high-degree basis.
Moreover, when we aggregate predictions over an ensemble of polynomial representations,
our prototype demonstrates the natural increase in uncertainty we intuitively associate with extrapolation.

Our second algorithm samples ensembles of decision trees that are constructed using the parsimonious inference objective.
These models aim to predict discrete labels through a sequence of partitions on continuous feature coordinates.
Our random forest prototype demonstrates the ability to learn credible prediction uncertainty from extremely small and heavily skewed datasets,
which we contrast with a standard decision tree model and bootstrap aggregation.
Both of these algorithms achieve superior prediction uncertainty through Bayesian inference from
prior belief that is derived to both quantify and naturally suppress complexity over arbitrary explanations. 

\tAdd{Although our basic hyperprior is subject to a choice of interpreter---which need not be Turing-complete, but must transform valid codes into coherent probability distributions over predictive models---we go on to show that to be consistent with this theory, there exists a unique hyperprior over an ensemble of Turing-complete interpreters, \Cref{thm:utm_ensemble}.}


\subsection{Organization}

\Cref{sec:background} begins with a discussion and illustration of the severe inadequacies of traditional machine learning training approaches that depend on cross-validation.  
We then briefly review the critical connections between scientific principles, rational belief, and Bayesian inference,
which provide a sound theory to obtain rigorously justified uncertainty in predictions.
When placed in the machine learning context, however, we explore how principled justification for prior belief over abstract models, as well as our unavoidable disregard for an infinite number of alternative models, remains a critical challenge.
Further, we summarize how our theory of information is derived to satisfy key properties that allow us to relate the various forms of complexity that follow in the parsimonious inference objective.

\Cref{sec:complexity} continues with our main contributions, including a discussion of generalized description length, a coherent complexity hyperprior, and the principles of minimum information and maximum entropy.
These notions culminate in the parsimonious inference objective, providing a suitable framework to understand and control model complexity over arbitrary learning architectures.
We also show how this objective allows us to quantify memorization.
Our theory allows us to apply these concepts within a wide variety of approaches to solve learning problems, including variational inference techniques.

\Cref{sec:implementation} examines implementation details within our prototype algorithms, including efficient encodings and training strategies for polynomial regression and decision trees.

\Cref{sec:discussion} concludes with a discussion of \tAdd{how we may consistently compare multiple interpreters}, a pathway to frame and address computability \`a priori, our theory's relationship to other work, and a summary of our findings.


\section{Background}
\label{sec:background}

In order to clarify how we may improve trust in machine learning predictions, we must begin with the origin of trust in science.
\tAdd{The epistemological foundation of the scientific method shares a fundamental connection with Bayesian inference and determines how we may optimally account for evidence to learn plausible explanations.
Bayesian theory alone, however, does not provide a complete learning framework when we employ high-parameter families, such as most machine learning architectures.
Thus, we also review Solomonoff's and Kolmogorov's notions of complexity as a means to promote simplicity in learned models.
To motivate the need for this discussion, we begin by illustrating the severe deficiencies of standard machine learning training practices when they are applied to small datasets.}


\subsection{Memorization}

The term \textit{memorization} is often conflated with \textit{overtraining}, but we distinguish these terms as follows.
Overtraining is characterized by degradation in prediction quality on unseen data that occurs after an initial stage of improvements.
In contrast, memorization refers more generally to any predictive algorithm that exhibits unjustifiable confidence, or low prediction uncertainty, in the training dataset labels that were used to adjust model parameters.
\Cref{sec:info_min} provides rigorous analysis to justify this view.
Conflating these terms leads to an incorrect picture of the problem; to avoid memorization, we must merely halt training at the correct moment.

Machine learning algorithms are typically trained using some variation of stochastic gradient descent \citep{Robbins1951}.
When applied to overparameterized models, traditional optimization strategies are subject to overtraining.
Cross-validation \citep{Allen1974} attempts to prevent overtraining by monitoring predictions on a holdout dataset, but we show how this method still fails to prevent memorization on small datasets.
The same strategy is also used to tune hyperparameters, such as such as regularization weights and learning-rate schedules.

The obvious difficulty presented by cross-validation is the inherent tradeoff between using as much data as possible to train parameters, but also having a reliable estimator for prediction quality.
For limited data, standard practices apply some form of k-fold cross-validation \citep{Hastie2009}.
One forms k distinct partitions of the dataset, trains k models respectively, and aggregates predictions by averaging.
Leave-one-out cross-validation uses the same number of partitions as datapoints.
Each partition reserves only one observation to estimate the best model over each training trajectory.

\begin{figure}[h!]
	\centering
	\includegraphics[width=0.825\textwidth]{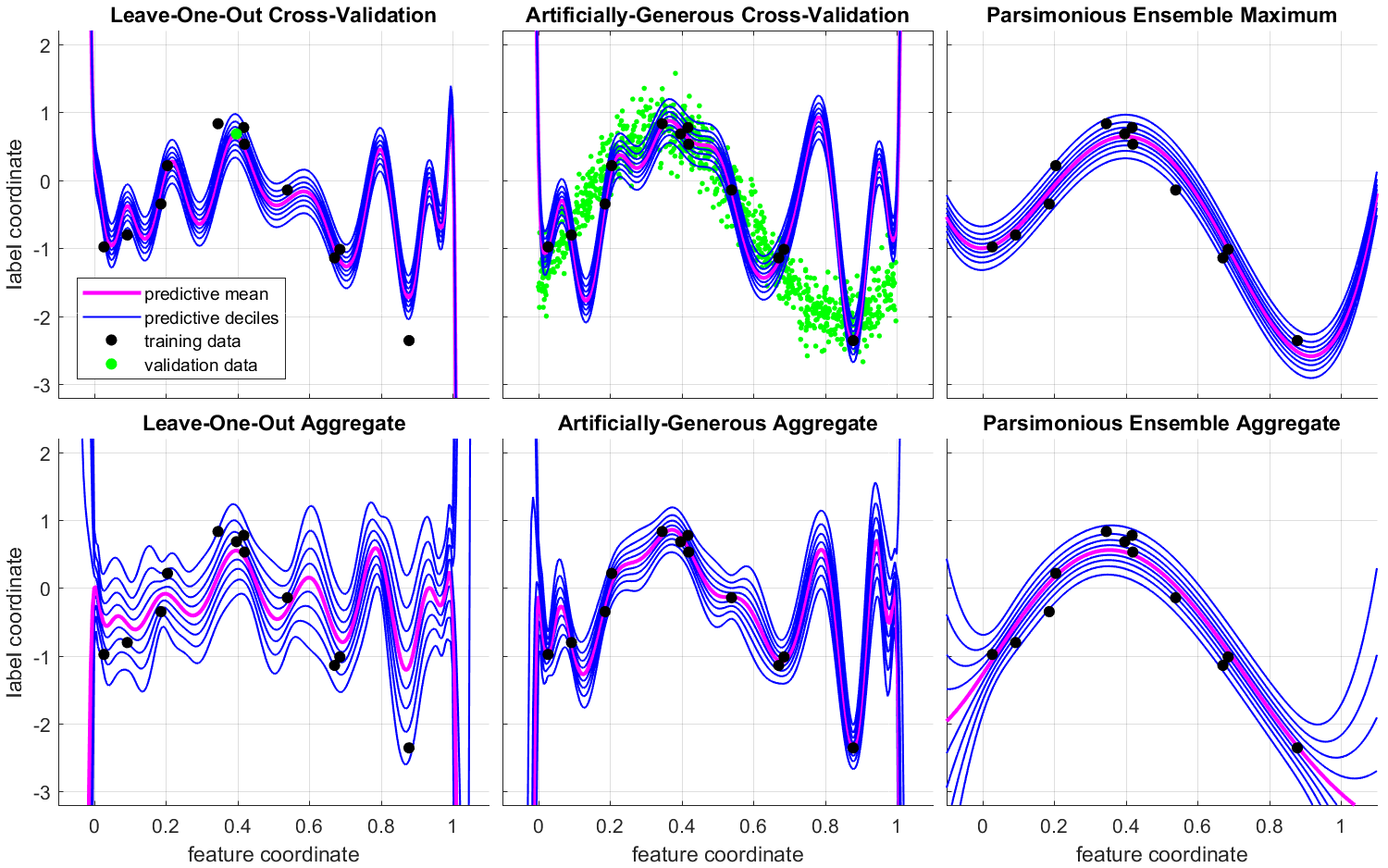}
	\caption{\small
	Illustration of standard training shortcomings.
	Top-left: optimum obtained from holding out the green point.
	Bottom-left: mean predictions over all single-holdout optima.
	Top-middle: idealized training on all original data and 1000 extra validation points from ground truth.
	Bottom-middle: mean predictions over 12 random starts, $\mathcal{N}(\theta_i \mid \mu=0,\sigma = 0.2)$.
	Removing the tradeoff between holdout and training data does not prevent complexities in predictions.
	Top-right: optimal model discovered using our theoretical framework and prototype algorithm.
	Bottom-right: our aggregate accounts for many plausible models, improving robustness and demonstrating natural extrapolation uncertainty.
	}
  \label{fig:deficiencies}
\end{figure}

\Cref{fig:deficiencies} demonstrates these techniques using polynomial regression, fitting 20th degree polynomials with only 12 points.
\tAdd{Retaining more polynomial coefficients than training points allows us to observe how standard training fails when data are limited.}
The top-left shows an example of a single model trained by holding out one point for validation, shown in green.
The bottom-left shows average predictions over 12 such models.
Yet, suppose we could train with all 12 points while remaining highly confident that we will halt training at the correct moment.
This ideal is demonstrated as a thought experiment in the middle column of \Cref{fig:deficiencies} by sampling 1000 extra data from the \textit{generative process}, the ground truth mechanism that creates observations.
We see that eliminating the tradeoff between training and validation would not prevent artifacts from developing that confidently hew to scant observations, memorization.

\tAdd{Typically, one would also use cross-validation to tune the optimal polynomial degree, which would certainly constrain complexity somewhat.
The purpose of this experiment, however, is to show that na\"ive training may never even explore low-complexity models, especially in high dimensions.
For most high-parameter families, such as neural networks, there is no feasible hierarchy of bases that would be analogous to limiting the polynomial degree.
For example, if we wished to constrain model parameters to a fixed sparsity pattern, the number of patterns to test would grow exponentially in the number of nonzero elements.}

This experiment demonstrates how neither of the competing cross-validation objectives address the core problem with learning from limited data.
Memorization is often framed in terms of a bias-variance tradeoff;
predictions should avoid fluctuating rapidly, but also remain flexible enough to extract predictive patterns.
In our theoretical framework, however, memorization is more comprehensively and rigorously understood as unparsimonious model complexity, i.e.~increases in model information that are not justified by only small improvements to training predictions.

Regularization strategies attempt to address this heuristically by penalizing excessive freedom in learning parameters, for example attaching an $\ell_1$ or $\ell_2$ norm to the training objective.
While many of these approaches can be equivalently cast as choices of prior belief, they lack \tAdd{a unifying principle that would illuminate and resolve choices of regularization shape and weight.
One must, again, resort to hyperparameter tuning via cross-validation, thus failing to address the core challenge: to efficiently learn from limited data.}


\subsection{Scientific Reasoning and Bayesian Inference}

In order to reiterate the concrete relationship between Bayesian inference and scientific reasoning,
we review the epistemological foundations of reason at the center of the scientific method.
These foundations bear decisive consequences regarding the valid forms of analysis we may pursue in order to obtain rational predictions.
At its core, the scientific method relies on coherent mathematical models of observable phenomena that have been informed over centuries of physical measurements.
Within the field of epistemology, this is the naturalist view of rational belief \citep{Brandt1985}.
It holds that validity is ultimately derived from consistency, which can be understood in three key components:
\begin{enumerate}
\item Rational beliefs must be logical, avoiding internal contradictions;
\item Rational beliefs must be empirical, accounting for all available evidence;
\item Rational beliefs must be predictive, continually reassessing validity by how well predictions agree with new observations.
\end{enumerate}
The third point is really nothing more than a restatement of the second point, placing emphasis on the evolving nature of rational beliefs as new data become available.
The critical significance of the first point is that it provides a path to elevate the second and third points to a rigorous extended logic: Bayesian inference.

Building on the rich body of work by many scholars---including \citet[original 1926]{Ramsey2016}, \citet{DeFinetti1937}, and \citet[original 1939]{Jeffreys1998}---\citet{Cox1946} shows that for a mathematical framework analyzing degrees of truth, belief as an extended logic,
to be consistent with binary propositional logic, that formalism must satisfy the laws of probability:
\begin{enumerate}
\item Probability is nonnegative.
\item Only impossibility has probability zero.
\item Only certainty has maximum probability, normalized to one.
\item To revise the degree of credibility we assign to a model upon reviewing empirical evidence, we must apply Bayes' theorem.
\end{enumerate}
Consequently, the Bayesian paradigm provides a uniquely rigorous approach to quantify uncertainty in predictions derived through inductive reasoning.
Therefore, the only logically correct path to quantify and suppress memorization in learning must be cast within the Bayesian perspective.

Analysis proceeds with a probability distribution called the prior $\pP(\vTheta)$ that quantifies our lack of information, or initial uncertainty, in plausible explanatory models.
Here, $\vTheta$ is any specific parameter state within a model class, or computational architecture.
When we need to emphasize the prior's dependence on a model class, as well as the shape of the parameter distribution within that class, we will write the prior as $\pP(\vTheta \mid \vPsi)$,
where a description $\vPsi$, or hyperparameter sequence, provides such details.
We will examine how $\vPsi$ plays a key role regarding model complexity in detail in \Cref{sec:complexity}.
The empirical data are expressed as a set of ordered pairs $\sData=\{ (\vX_i, \vY_i) \mid i \in [n] \}$ that have been sampled from the generative process $\pG(\vX, \vY)$.
Features $\vX_i$ are used to predict labels $\vY_i$ from the architecture paired with $\vTheta$.
We write the predicted distribution over all potential labels as $\pP(\vY_i \mid \vX_i, \vTheta)$.

If the ordered pairs in $\sData$ represent independent samples from the underlying process,
the likelihood is evaluated as $\pP(\sData \mid \vTheta) = \prod_{i\in [n]} \pP(\vY_i \mid \vX_i, \vTheta)$, which expresses the probability of observing $\sData$ if a hypothetical explanation $\vTheta$ held.
Then, we update our beliefs according to Bayes' theorem.
In our picture, we hold that having $\vTheta$ alone is sufficient to evaluate predictions.
When we explicate the role of prior descriptions $\vPsi$, that means inference can be written as
\submitBlank
\begin{align*}
\pP(\vTheta \mid \sData, \vPsi) = \frac{\pP(\sData \mid \vTheta) \pP(\vTheta \mid \vPsi)}{\pP(\sData \mid \vPsi)}
\quad\text{where}\quad
\pP(\sData \mid \vPsi) =  \integ{d\vTheta} \pP(\sData \mid \vTheta) \pP(\vTheta \mid \vPsi)
\end{align*}
is the model-class evidence.
If we have a hyperprior $\pP(\vPsi)$ over potential descriptions, we can also infer the hyperposterior
\submitBlank
\begin{align*}
\pP(\vPsi \mid \sData) = \frac{\pP(\sData \mid \vPsi) \pP(\vPsi)}{\pP(\sData)}
\quad\text{where}\quad
\pP(\sData) =  \integ{d\vPsi} \pP(\sData \mid \vPsi) \pP(\vPsi).
\end{align*}

The central point of inference is that it does not attempt to identify a single explanation matching the data, as with stochastic gradient optimization and cross-validation.
Rather, inference naturally adheres to the Epicurean principle---we should retain multiple explanations according to their respective degrees of plausibility---within a coherent mathematical framework.
\tAdd{As a distribution, the posterior is meaningful in a way that a single model is not; it allows us to update our beliefs consistently as new evidence emerges.}
We obtain rational predictions by evaluating the posterior predictive integral, or even the hyperposterior predictive integral, respectively constructed as
\submitBlank
\begin{align*}
& \pP(\vY \mid \vX, \sData, \vPsi) = \integ{d\vTheta} \pP(\vY \mid \vX, \vTheta) \pP(\vTheta \mid \sData, \vPsi)
\quad\text{and}\\
& \pP(\vY \mid \vX, \sData) = \integ{d\vPsi} \pP(\vY \mid \vX, \sData, \vPsi) \pP(\vPsi \mid \sData).
\end{align*}
The resulting predictions meet the exigent standard of rational belief for meaningful uncertainty quantification.


\subsection{The Universal Scope of Prior Belief}

The pervasive objection to the Bayesian paradigm is the lack of clear provenance for prior belief.
In addition to objections based on subjectivity, translating our intuitive beliefs into distributions can be difficult.
This problem is exacerbated in the domain of machine learning, where computational models are abstract and driven only by practical utility, rather than well-understood physical principles.
The premise of machine learning is that we do not need to integrate expert knowledge and specialized scientific theory into algorithms to obtain useful predictions from data,
which eludes the traditional view of priors, that we must express our beliefs.

Prediction sensitivity to prior belief is most apparent when the number of parameters approaches or exceeds the size of our dataset, as illustrated in \Cref{fig:sensitivity}.
If we have $n$ observations and $k > n$ differentiable parameters, then every point in parameter space must have at least $k-n$ perturbable dimensions in which the likelihood gradient is zero.
Because the likelihood remains constant as we move through these dimensions, each point lives within a $(k-n)$-dimensional submanifold wherein prior belief entirely determines the structure of posterior belief.
Thus, within these submanifolds, the contribution of parameter uncertainty to prediction uncertainty is not affected by evidence.
Clearly, we cannot be satisfied with meeting only the bare conditions for technically rational belief;
we require concrete philosophical justification for prior belief.

\begin{figure}[b!]
	\centering
	\includegraphics[width=\textwidth]{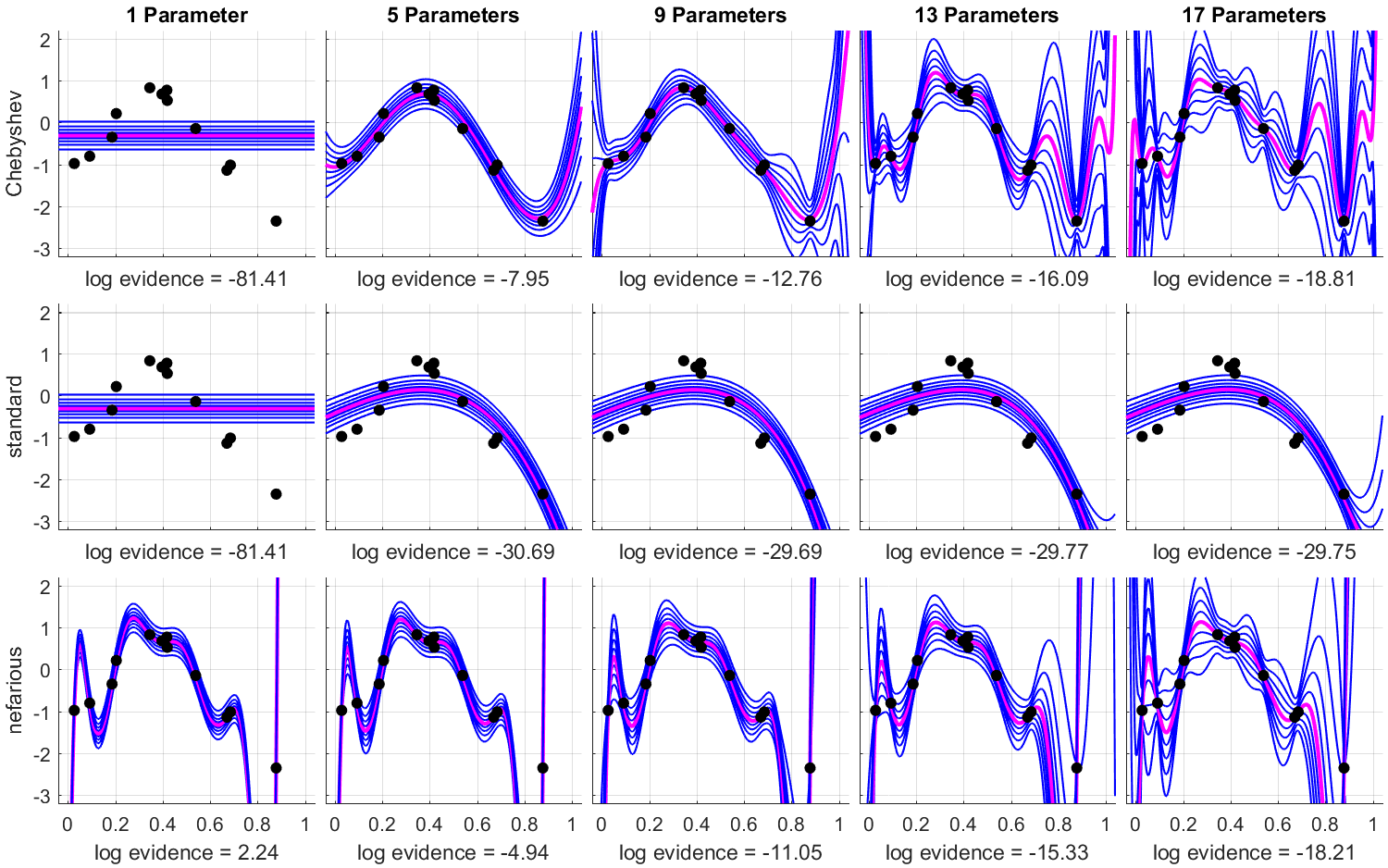}
	\caption{\vspace{-0.1in}\small
	Illustration of prediction sensitivity to prior belief.
	The first row uses Chebyshev polynomial bases and the second uses standard bases.
	The third row exacerbates the problem with basis-dependence by using a polynomial that already memorized the data for the first basis function, followed by the Chebyshev basis.
	All priors are normal, $\pN(\vTh \mid 0, \mI)$.
	The choice of basis clearly matters; inference alone does not prevent the development of artifacts we associate with memorization.
	}
  \label{fig:sensitivity}
\end{figure}

In order to appreciate the solution, we must grasp the full severity of the problem by framing it in the most arduous scope.
We can define the \textit{model universe} as the set containing every computational architecture that could produce coherent predictions over $\vY$ from $\vX$.
By considering inference over the model universe, we see that every architectural design decision is equivalent to a choice of support for prior belief, i.e.~the subdomain in which prior belief is nonzero.
Using model-class descriptions $\vPsi$ to capture potential choices of prior belief allows us to subsume these complications by investigating the correct form of a hyperprior $\pP(\vPsi)$.

This picture also illuminates a second significant challenge in the machine learning setting, the problem of dimensionality in high-parameter families.
Even if we obtain an attractive hyperprior, we can always construct increasingly complicated architectures.
It is not possible to explore all of them.
Occam's Razor provides a compelling path to a solution;
explanations should not exhibit more complexity than what is required to explain the evidence.
We conclude that a comprehensive hyperprior must compute and suppress a rigorous formulation of complexity.
Moreover, our learning framework must justify disregarding infinite dimensions from inference and simultaneously address how to feasibly construct or approximate restricted posteriors.


\subsection{Controlling Complexity}

Bayesian theorists have had a persistent interest in articulating core principles for constructing priors over abstract models, particularly within the objectivist Bayesian philosophy.
Examples include maximum entropy priors \citep{Jaynes1957,Good1963} and Jeffrey's priors \citep{Jeffreys1946}.
Other approaches use information criteria to determine a suitable number of parameters, such as the Akaike Information Criterion (AIC)~\citep{Akaike1974} and Bayesian Information Criterion (BIC)~\citep{Schwarz1978}.
In contrast to these approaches that explicitly compare model classes, in which a parameter is either present or absent,
Automatic Relevance Determination (ARD) \citep{Mackay1995,Neal2012} takes a softer approach.
ARD uses a hyperprior to express uncertain relevance of different parameters and features in a model.
It postulates that most model parameters should be close to zero because only a limited number of features are relevant for prediction.
Through inference, relevant model parameters can be identified automatically.
Having been specifically developed for neural networks, ARD provides an important perspective to understand complete theories of learning.
We will discuss the relationship between ARD priors and our theory in \Cref{sec:ardprior}.

Perhaps the most principled approach to a universal prior is Solomonoff's work on algorithmic probability \citep{Solomonoff1960,Solomonoff1964a,Solomonoff1964b,Solomonoff2009}.
Solomonoff derives a prior over all possible programs based upon their lengths using binary encodings subject to an optimal UTM.
\citet{Hutter2007}, reviewing central principles of reason, goes further to show how Solomonoff's framework solves important philosophical problems in the Bayesian setting,
including predictive and decision-theoretic performance bounds under the assumption that the generative process is a program.
\citet{Potapov2012} also discuss Solomonoff's algorithmic probability, emphasizing the importance of retaining many alternative models to not only learn robust predictions,
but also maintain adaptability in decision making. 

Kolmogorov's work on mapping complexity \citep{Kolmogorov1965} is closely related and we will examine it in detail in \Cref{sec:kolmogorov}.
Rissanen's work on universal priors and Minimum Description Length (MDL) \citep{Rissanen1983,Rissanen1984} is also related.
We examine his universal prior on integers in \Cref{sub:encodings} and the relationship between our theory and MDL in \Cref{sub:mdl}.
The key advantage of Solomonoff's approach is that it applies generally to any model we can program, thus eliminating artificial constraints on computational architectures.
Solomonoff does not separate the model $\vTheta$ from the model class $\vPsi$, since any model from any model class can be expressed as a program.
As this encoding-length based prior provides a strong base for our work, we provide a detailed discussion in \Cref{sec:solomonoff}.

\tAdd{Information-theoretic formulations of complexity can be traced back to \citet{Shannon1948} and the concept of entropy as a measure of the uncertainty associated with sequences of discrete symbols that may be transmitted over a communication channel.}
In order to rigorously understand how information in our datasets relates to Bayesian inference and encoding complexity,
we developed a theory of information \citep{Duersch2020} rooted in understanding information as an expectation over rational belief.

Given an arbitrary latent variable $\vZ$, we would like to measure the information gained by shifting belief between hypothetical states, i.e.~from $\pQ_0(\vZ)$ to $\pQ_1(\vZ)$.
We require this measurement to be taken relative to a third state of belief, $\pR(\vZ)$, which we hold to be valid.
The precise reasoning by which validity of $\pR(\vZ)$ is derived is an important epistemological question.
For our purposes, $\pR(\vZ)$ will either be rational belief, expressing our present understanding of the actual state of affairs, or a choice, representing a hypothetical state of affairs following a decision.

Rational belief is defined as the posterior distribution resulting from inference, which reserves some nonnegative probability for every outcome that is plausible.
In contrast, we regard a choice as a restriction on the support of belief, effectively confining probability to any distribution that we can describe.
This occurs when we must adhere to a course of action from a set of mutually incompatible options.
The following postulates and \Cref{thm:info} summarize key results.

\begin{enumerate}
\item Information gained by changing belief from $\pQ_0(\vZ)$ to $\pQ_1(\vZ)$ is quantified as an expectation over a third state $\pR(\vZ)$, called the view of expectation.
\item Information is additive over independent belief processes.
\item If belief does not change then no information is gained, regardless of the view of expectation.
\item Information gained from any normalized prior state of belief $\pQ_0(\vZ)$ to an updated state of belief $\pR(\vZ)$ in the view of $\pR(\vZ)$ must be nonnegative.
\end{enumerate}

\begin{theorem}
\label{thm:info} \thmTitle{Information as a rational measure of change in belief.}
Information, measured in bits, satisfying these postulates must take the form
\submitBlank
\begin{displaymath}
\info{\pR(\vZ)}{\pQ_1(\vZ)}{\pQ_0(\vZ)} = \integ{d\vZ} \pR(\vZ) \log_2\left( \frac{\pQ_1(\vZ)}{\pQ_0(\vZ)} \right) \text{bits}.
\end{displaymath}
\end{theorem}

When the view of expectation is the same as the target belief, we recover the Kullback-Leibler divergence \citep{Kullback1951} 
\submitBlank
\begin{align*}
\KL{\pR(\vZ)}{\pQ(\vZ)} = \info{\pR(\vZ)}{\pR(\vZ)}{\pQ(\vZ)}. 
\end{align*}
The entropy of a distribution $\pP(\vZ)$ over discrete outcomes $\vZ \in \{ \vZ_i \mid i \in [n] \}$ is equivalent to the expected information gained upon realization in the view of the realization
\submitBlank
\begin{displaymath}
\entropy{\pP(\vZ)} = \sum_{i=1}^n \pP(\vZ_i) \log_2\!\left( \frac{1}{\pP(\vZ_i)} \right) \text{bits}.
\end{displaymath}

Our theory allows us to relate and analyze changes in belief regarding our data, model parameters, and hyperparameters within a unified framework.
\Cref{sec:corollaries} provides selected corollaries of \Cref{thm:info} for reference. 

Our main result, \Cref{thm:parsimony_optimization}, builds on this work to show how memorization may be quantified, \Cref{cor:memorization}, and prevented.
Simple changes to the model structure that benefit multiple predictions are parsimonious, worthwhile investments.
In contrast, memorization as a wasteful transfer of information from the space of predictions to the space of explanations.
As we must inevitably solve feasible approximations of the posterior predictive integral in order to obtain practical predictions, our theory provides additional benefit by allowing us to analyze posterior approximations within the same formalism.


\section{Complexity and Parsimony}
\label{sec:complexity}

We present our theoretical learning framework in three parts.
First, we analyze a modest generalization of Kolmogorov's notion of program length to sequences of symbols drawn from arbitrary alphabets that may be conditioned on previously realized symbols.
This simplifies our ability to assign complexity to arbitrary descriptions of prior belief.
Second, we use description length to derive a hyperprior that extends Solomonoff's formulation of algorithmic probability to general inference architectures.
Third, we cast learning as an information minimization principle.
We show how learning balances the two forms of information contained within models, due to both prior descriptions and inference, against the information the models provide about our dataset.
Not only does this formalism allow us to analyze the utility of potential restrictions of prior belief, we also recover variational inference optimization from the same principle.


\subsection{Program Length and Kolmogorov Complexity}
\label{sec:kolmogorov}

\newSet<\sObj>{A}
\newSet<\sSubObj>{B}
\DeclareDocumentCommand \sSub{m}{\sSubObj\!\left[ #1 \right]}
\newVector<\vObj>{a}
\newVector<\vObjr>{\check{a}}
\newMatrix<\sAlp>{\Sigma}
\newSet<\sX>{X}
\newSet<\sY>{Y}

Kolmogorov's discussion of complexity begins with a countable set of objects $\sObj = \{ \vObj \}$ that are indexed with binary sequences.
For Kolmogorov, an object $\vObj$ is a program and the length of the program $\fLen(\vObj)$ is taken to be the number of binary digits in a corresponding binary sequence $\vPsi(\vObj)$.
Given a domain of program inputs $\sX$ and a codomain of outputs $\sY$,
a \textit{programming method} $\vPhi(\cdot, \cdot)$ accepts the program $\vObj$ and an input $\vX \in \sX$ and returns an output $\vY = \vPhi(\vObj, \vX) \in \sY$.
The Kolmogorov complexity of an ordered pair $(\vX, \vY) \in \sX \times \sY$ is the length of the shortest program that is capable of reproducing the pair
\submitBlank
\begin{align*}
\fKol_{\vPhi}(\vX, \vY) = \min_{\vObj\in\sSubObj}\, \fLen(\vObj)
\quad\text{where}\quad
\sSubObj = \left\{ \vObj \mid \vY = \vPhi(\vObj, \vX) \right\} \subset \sObj.
\end{align*}
An ordered pair may be understood to enforce multiple function values or even the entire mapping that defines a function.
Further, Kolmogorov's framework easily captures the complexity of a singleton $\vY$ by taking an empty input, $\vX=\emptyset$.

Because Turing-complete programming methods can simulate one another,
\tAdd{the shortest length of a program in a new language is bound from above by that of the source language plus a constant;
the new shortest program cannot be longer than simply attaching a simulator to the source version.
Thus, the length of an efficient simulator for the original programming method provides the bounding constant offset.}

The descriptions of interest to us, however, may not admit perfectly efficient binary codes.
For example, we may wish to represent the outcome of rolling of a balanced six-sided die, \tAdd{having an approximate entropy of $2.585$ bits.}
Rather than solving for an optimal binary encoding \citep{Huffman1952},
\tAdd{yielding an expected length of approximately $2.667$ bits,}  
it is convenient to extend the notion of the length to finite sequences of symbols drawn from multiple alphabets.
\tAdd{In this case, using an alphabet with six symbols would allow the encoding to exactly achieve the entropy limit.}
See \Cref{sub:decision_trees} for another example in which this extension supports efficient descriptions of feature domain partitions.

Let a description be composed of a sequence of symbols represented as $\vPsi = (\vS_i)_{i=1}^n$.
When we wish to draw attention to the role of a sequence as an encoding of an object, we write $\vPsi(\vObj)$.
For our purposes, these objects do not necessarily need to be programs.
Rather, they are simply descriptions of belief, $\pP(\vTh \mid \vPsi)$.
For each $i \in [n]$, a symbol $\vS_i$ is selected from an alphabet $\sAlp_i$.
We emphasize that each alphabet is allowed to depend on previously realized symbols in the sequence so that the sequence of alphabets is not fixed.
To avoid cumbersome notation and excessive indexing variables, 
we express subsequences as $(\vS)_{1}^{j}$ and leave the natural indexing $(\vS_i)_{i=1}^{j}$ implied.
Let $(\vS)_{1}^{0}$ indicate the empty subsequence.
We also leave the conditional dependence of each alphabet on previous symbols implied so that we may simply write $\sAlp_{i}$ rather than $\sAlp_{i}\!\left[ (\vS)_{1}^{i-1} \right]$.
As with Kolmogorov, the length of an object is derived from a sequence, $\fLen(\vObj) = \fLen(\vPsi(\vObj))$.

If we treat each symbol in an encoding $\vPsi(\vObj)$ as a discrete random variable, we have
\submitBlank
\begin{align*}
\pP(\vPsi(\vObj)) = \prod_{i=1}^n \pP(\vS_i \mid (\vS)_{1}^{i-1}).
\end{align*}
The entropy corresponding to each potential symbol, or the information we expect to gain upon realization,
is the maximum if and only if the probability of each symbol is uniform over its alphabet, $\pP(\vS_i \mid (\vS)_{1}^{i-1}) = \frac{1}{|\sAlp_i|}$.
That is, 
\submitBlank
\begin{align*}
\sum_{\vS_i \in \sAlp_i} \pP(\vS_i \mid (\vS)_{1}^{i-1}) \log_2\!\left(\frac{1}{\pP(\vS_i \mid (\vS)_{1}^{i-1})}\right) \leq \log_2(|\sAlp_i|).
\end{align*}
Our construction of generalized length in \Cref{def:length} invokes the principle of maximum entropy to remove restrictions on the kinds of codes we can consider,
while recovering Kolmogorov's length when all symbols are binary digits.

\begin{definition}
\label{def:length} \thmTitle{Generalized Length as Maximum Entropy Encoding.}
The generalized length of an arbitrary sequence $\vPsi$ is the upper bound on entropy of the corresponding sequence of alphabets from which each symbol is drawn,
\submitBlank
\begin{displaymath}
\fLen(\vPsi) = \sum_{i=1}^n \log_2( |\sAlp_i| ) \,\text{bits}.
\end{displaymath}
\end{definition}

\Cref{cor:length_lowerbound} shows that this length saturates the information lower bound in Shannon's source coding theorem \citep{Shannon1948}.
We provide all proofs in \Cref{sec:proofs}.
Furthermore, when we develop an efficient encoding for the kinds of objects we would like to use,
\Cref{cor:probability_from_length} allows us to naturally derive the probability of an object from the encoding.

\begin{corollary}
\label{cor:length_lowerbound} \thmTitle{Generalized Length Lower Bound.}
Given a set of objects $\sObj$ and probabilities $\pP(\vObj)$ for all $\vObj \in \sObj$,
the expected generalized length of an object is bound from below by the entropy
\submitBlank
\begin{displaymath}
\expect_{\pP(\vObj)} \fLen(\vPsi(\vObj)) \ge \expect_{\pP(\vObj)} \log_2\!\left(\frac{1}{\pP(\vObj)}\right).
\end{displaymath}
\end{corollary}

\begin{corollary}
\label{cor:probability_from_length} \thmTitle{Probability from Length.}
Given a set of objects $\sObj$ and a maximum entropy encoding with $\pP(\vObj) = \pP(\vPsi(\vObj))$ for all $\vObj \in \sObj$,
the generalized length satisfies
\submitBlank
\begin{displaymath}
\pP(\vObj) = 2^{-\fLen(\vObj)}.
\end{displaymath}
\end{corollary}


\subsection{Algorithmic Probability}
\label{sec:solomonoff}

\DeclareDocumentCommand \pQObj{} {\pQ(\vObj)}
\DeclareDocumentCommand \pQsObj{} {\pQs(\vObj)}

\tAdd{Five years before Kolmogorov published his work on mapping complexity}, Solomonoff articulated the foundations for inductive inference and algorithmic probability.
He was specifically interested in programs capable of reproducing a binary sequence $\vY$,
i.e.~the subset of programs $\sSubObj = \left\{ \vObj \mid \vY = \vPhi(\vObj, \emptyset) \right\} \subset \sObj$,
and he derived the probabilistic contribution of each program to plausible continuations of the sequence
\submitBlank
\begin{align*}
\pP(\vObj \mid \vY) \propto \begin{cases}
 2^{-\fLen(\vPsi(\vObj))} & \vObj \in \sSubObj \\
 0 & \vObj \notin \sSubObj
\end{cases}
\end{align*}
where, as with Kolmogorov's picture, length corresponds to an optimal binary encoding, subject to \tAdd{an optimal UTM.
That is, the UTM for which the optimal binary encoding is shortest.}
We understand his result as Bayesian inference wherein \Cref{cor:probability_from_length} provides prior belief and a program has unit likelihood if \tAdd{it halts and} reproduces the sequence.
\tAdd{Otherwise, the likelihood is zero.}
It follows that the Kolmogorov complexity is simply the \tAdd{length of the Maximum A Posteriori (MAP)} estimator in the same picture.
As such, the Kolmogorov complexity is the minimum amount of information that is possible to gain
by restricting belief to a discrete program that is capable of reproducing a desired ordered pair.
If, however, we allow distributions of belief over many programs, so that $\pQObj \geq 0$ for any $\vObj \in \sSubObj$,
\Cref{cor:solomonoff} shows that Solomonoff's algorithmic probability is the minimizer of information gain, improving beyond the Kolmogorov complexity.
 
\begin{corollary}
\label{cor:solomonoff} \thmTitle{Information Optimality of Solomonoff Programs.}
If we measure the change in belief from all possible programs according to \Cref{cor:probability_from_length}
to any distribution $\pQObj$ that restricts belief to programs capable of reproducing an input-output pair $(\vX, \vY)$,
the minimizer
\submitBlank
\begin{align*}
& \pQsObj = \argmin_{\pQObj}\, \KL{\pQObj}{\pP(\vObj)}
\quad\text{subject to}\quad
\pQObj = 0
\quad\forall\quad
\vObj \notin \sSubObj = \left\{ \vObj \mid \vY = \vPhi(\vObj, \vX) \right\},
\end{align*}
is uniquely given by
\submitBlank
\begin{align*}
& \pQsObj = \frac{2^{-\fLen(\vPsi(\vObj))}}{\pP(\sSubObj)}
\quad\forall\quad
\vObj \in \sSubObj
\quad\text{where}\quad
\pP(\sSubObj) = \sum_{\vObj \in \sSubObj} 2^{-\fLen(\vPsi(\vObj))}.
\end{align*}
\end{corollary}

Solomonoff's picture is even more general than it first appears.
There is no need to confine our attention to binary programs that reproduce binary sequences.
We may apply the same framework to any algorithm that generates coherent probabilities on a given dataset, $\pP(\vY \mid \vX, \vObj)$.
Since any algorithm we write is ultimately still a program, this induces universal prior belief over arbitrary predictive algorithms.
\tAdd{Again, such a prior depends on the choice of UTM or programming method.}
Then, Bayesian inference yields rational belief as a posterior distribution over all such algorithms.

Yet, this approach is fundamentally difficult because it requires us to efficiently explore the posterior over suitable programs via their discrete sequences.
Since it is not always possible to anticipate how a program will respond to given inputs in finite time, we arrive at the problem of uncomputability.
Moreover, even if we discover seemingly high posterior programs, we cannot guarantee that our sample adequately approximates the posterior predictive integral.
We discuss this problem and a potential solution further in \Cref{sec:computability}.

Rather than restricting our attention to programs, we would like to allow more general inference architectures.
As indicated earlier, we accomplish this by relaxing $\vPsi$ to be merely a description of prior belief.
\tAdd{Doing so requires an \textit{interpreter}, which translates valid sequences into coherent distributions over valid models, $\pP(\vTh \mid \vPsi)$,
and then computes coherent likelihoods from specific models, $\pP(\vY \mid \vX, \vTh)$.
If the interpreter is Turing complete, then a description could still be a complete program,
but we can also consider interpreters that merely require specification of a few hyperparameters of a distribution that is well-suited to our data source.}

Not only are short descriptions easier to discover, letting our data drive updates in belief through inference is often substantially more information-efficient than accounting for a full program that computes the equivalent result.
In this picture, we lose the ability of interpreters to simulate one another, but we gain access to simpler encodings that may be much easier to propose and evaluate.
That said, the choice of interpreter is an important problem.
We revisit this issue in \Cref{sec:interpreters}.

We remark that subjective prior beliefs may be expressible by allowing symbol probabilities to be nonuniform over relevant alphabets.
In this view, just as generalized length is the limit of expected information gained by realization from an encoding,
the corresponding probability in \Cref{cor:probability_from_length} may be regarded as the limit of subjective priors within a given encoding.
Subjective prior beliefs are also reflected by the choice of interpreter.

The parsimonious hyperprior over corresponding model classes from \Cref{cor:probability_from_length},  $\pP(\vPsi) = 2^{-\fLen(\vPsi)}$,
is enough to complete the Bayesian framework with well-founded justification for how we arrive at prior belief over arbitrary architectures.
When we perform Bayesian inference from a parsimonious prior or hyperprior, we call the result parsimonious rational belief.
To achieve computational feasibility, however, we still need to investigate principled restrictions of belief.


\subsection{The Principle of Information Minimization}
\label{sec:info_min}

We develop this paradigm in order to promote computational feasibility while retaining well-founded theoretical justification for resulting predictions.
As alluded to in \Cref{cor:solomonoff}, we can cast learning as an information minimization problem over our total change in belief due to observing the training data $\sData$,
selecting one or more model classes $\vPsi$, and solving for distributions over models $\vTheta$ within each class.
The principle of minimum information \citep{Evans1969}, based on the closely related principle of maximum entropy \citep{Jaynes1957},
intuitively states that driving the information gained upon viewing the training data to be as low as possible, we obtain better predictions.
If a dataset contains a highly predictive pattern, then once that pattern is known we can obtain strong predictions.
As a consequence, the information gained by observing new labels drops.
In contrast, the information gained by observing new labels will remain high when there is no discernable pattern.
We derive \Cref{thm:parsimony_optimization} from a rigorous formulation of the minimum information principle using our work regarding information as a rational measure of change in belief and show how this information objective can be manipulated into three terms
that provide insight into how we may understand and control complexity during learning as a constrained optimization problem.

\DeclareDocumentCommand \pRY{} {\pR(\vY \mid \vYr)}
\DeclareDocumentCommand \pQPsi{} {\pQ(\vPsi)}
\DeclareDocumentCommand \pQsPsi{} {\pQs(\vPsi)}
\DeclareDocumentCommand \pQTh{} {\pQ(\vTh)}
\DeclareDocumentCommand \pQThCPsi{} {\pQ(\vTh \mid \vPsi)}
\DeclareDocumentCommand \pQsThCPsi{} {\pQs(\vTh \mid \vPsi)}
\DeclareDocumentCommand \pQThPsi{} {\pQ(\vTh, \vPsi)}
\DeclareDocumentCommand \pQsThPsi{} {\pQs(\vTh, \vPsi)}
\DeclareDocumentCommand \pQsY{} {\pQs(\vY)}

\begin{theorem}
\label{thm:parsimony_optimization} \thmTitle{Parsimonious Inference Optimization.}
Let our training dataset be represented as an ordered pair $(\vX, \vY)$.
A model $\vTh$ computes coherent probabilities over potential labels $\vY$ from features $\vX$ as $\pP(\vY \mid \vX, \vTh)$.
Shorthand $\pP(\vY \mid \vTh)$ leaves dependence on $\vX$ implied.
A description of the model class is represented by a sequence $\vPsi$, which restricts prior belief to $\pP(\vTh \mid \vPsi)$.
The parsimonious hyperprior over potential sequences induced by generalized length is $\pP(\vPsi) = 2^{-\fLen(\vPsi)}$.
Our joint belief in labels, models, and prior encodings is given by $\pP(\vY, \vTh, \vPsi) = \pP(\vY \mid \vTh) \pP(\vTh \mid \vPsi) \pP(\vPsi)$.
Viewing the realized labels $\vYr$ from the dataset changes rational belief to $\pRY$, a distribution assigning full probability to the observed outcomes.
Let any potential choice of belief over descriptions after viewing the data be $\pQPsi$.
Likewise, within a model class $\vPsi$, an arbitrary distribution over models is $\pQThCPsi$.
The total information gained is given by the Kullback-Leibler divergence
\submitBlank
\begin{align*}
\KL{\pRY \pQThCPsi \pQPsi}{ \pP(\vY \mid \vTh) \pP(\vTh \mid \vPsi) \pP(\vPsi)}.
\end{align*}
The minimizer of this objective is equivalent to the maximizer of the following parsimony objective, expressed in three parts
\submitBlank
\begin{align*}
\omega\tAdd{[\pQThCPsi, \pQPsi]} =
& \expect_{\pQThCPsi \pQPsi} \info{\pRY}{\pP(\vY \mid \vTh)}{\pP(\vY \mid \vTh_0)}
& &\text{(prediction information)} \\
& - \expect_{\pQPsi} \KL{\pQThCPsi}{\pP(\vTh \mid \vPsi)}
& -&\text{(inference information)} \\
& - \expect_{\pQPsi} \fLen(\vPsi) + \entropy{\pQPsi}, 
& -&\text{(description information)}
\end{align*}
where $\vTh_0$ anchors the predictive information measurement to any fixed baseline.
\end{theorem}

\tAdd{The notation $\pRY$ is intended to emphasize that our rational belief in plausible labels is totally restricted to what we have observed after viewing the evidence.
In contrast, the arguments to the optimization objective, $\pQThCPsi$ and $\pQPsi$, do not depend on the data until optimization constrains them.}

We hold that the view of expectation taken in \Cref{thm:parsimony_optimization} is valid because it represents the known labels and the actual distributions that will be used in practice to compute predictions.
Further, this construction of information, rather than the reversed divergence $\KL{\pP(\vY, \vTh, \vPsi)}{\pRY \pQThCPsi \pQPsi}$, is necessary to avoid multiple infinities;
for each description $\vPsi$ with $\pQPsi=0$, the reversed divergence is infinite.
Moreover, we cannot avoid eliminating an infinite number of such cases from consideration.

The first term, prediction information, is the expected information gained about training data resulting from our belief in explanations.
Both secondary terms, inference information and description information, account for model complexity.
Anchoring predictive information to any fixed model $\pP(\vY \mid \vTh_0)$ allows us to coherently interpret label information as that which is gained relative to $\vTh_0$.
Any fixed predictive distribution suffices, including the prior predictive
\submitBlank
\begin{align*}
\pP(\vY \mid \vX) = \integ{d\vPsi\,d\vTh} \pP(\vY \mid \vX, \vTh) \pP(\vTh \mid \vPsi) \pP(\vPsi). 
\end{align*}
However, because the prior predictive may be difficult (or impossible) to compute, it is much simpler to use a na\"ive model $\vTh_0$. 
If we disregard the role of $\vPsi$ and only account for prediction information and inference information from prior belief $\pP(\vTh)$, i.e.~from the first two terms of the parsimonious inference objective,
then we recover a form of the Bayesian Occam's Razor \citep{Mackay1992a},
\submitBlank
\begin{align*}
& \info{\pRY}{\pP(\vY)}{\pP(\vY \mid \vTh_0)} = \log_2\!\left( \frac{\pP(\vYr)}{\pP(\vYr \mid \vTh_0)} \right) \\
&= \expect_{\pP(\vTh \mid \vYr)} \log_2\!\left(\frac{\pP(\vYr \mid \vTh)}{\pP(\vYr \mid \vTh_0)}\right) - \KL{\pP(\vTh \mid \vYr)}{\pP(\vTh)},
\end{align*}
provided we use the exact posterior $\pP(\vTh \mid \vYr)$ in place of $\pQThCPsi$.
Otherwise, we recover a variational inference objective that is equivalent to maximizing the Evidence Lower Bound (ELBO).
The Bayesian Occam's Razor also reveals the tradeoff between the information that explanations provide about our data and the complexity of those explanations,
but it leaves the provenance of $\pP(\vTh)$ unaddressed.

The description information terms show how our theory subsumes the Principle of Maximum Entropy.
If we were to disregard the critical role that the parsimonious hyperprior plays in controlling complexity, i.e.~dropping expected length,
we would be left with an optimization objective that drives increases in entropy within our chosen distribution of descriptions.
Doing so, however, would mean that long and complicated descriptions would be just as plausible as short and simple descriptions, as long as they are equally capable of explaining the data.
Correctly accounting for description length completes a rigorous formulation of Occam's Razor.

\DeclareDocumentCommand \sFeasible{} {\mathcal{F}}

Critically, \Cref{cor:model_inference,cor:hyper_inference} show that unconstrained optimization recovers, and is therefore consistent with, Bayesian inference and parsimonious rational belief.
As demonstrated in \Cref{sub:decision_trees}, some prior beliefs facilitate exact inference and easily allow us to take $\pQThCPsi =\pP(\vTh \mid \vYr, \vPsi)$.

\begin{corollary}
\label{cor:model_inference} \thmTitle{Optimality of Inference.}
Given a single description $\vPsi$ specifying prior belief $\pP(\vTh \mid \vPsi)$, the conditionally optimal distribution over models,
\submitBlank
\begin{align*}
\pQsThCPsi = \argmax_{\pQThCPsi}\quad
\expect_{\pQThCPsi} \info{\pRY}{\pP(\vY \mid \vTh)}{\pP(\vY \mid \vTh_0)} - \KL{\pQThCPsi}{\pP(\vTh \mid \vPsi)},
\end{align*}
is the posterior distribution $\pQsThCPsi = \pP(\vTh \mid \vYr, \vPsi)$.  
\end{corollary}

\begin{corollary}
\label{cor:hyper_inference} \thmTitle{Optimality of Hyper Inference.}
Applying the optimizer from \Cref{cor:model_inference} to the objective in \Cref{thm:parsimony_optimization} produces the second optimization problem
\submitBlank
\begin{align*}
\pQsPsi = \argmax_{\pQPsi}\quad
\expect_{\pQPsi} \log_2\!\left( \frac{\pP(\vYr \mid \vPsi)}{\pP(\vYr \mid \vTh_0)} \right) - \KL{\pQPsi}{\pP(\vPsi)}.
\end{align*}
The optimizer is the hyperposterior distribution, $\pQsPsi = \pP(\vPsi \mid \vYr)$.  
\end{corollary}

\tAdd{Yet, unconstrained optimization of $\pQPsi$, subject to a Turing-complete interpreter, would need to explore unlimited varieties of programs and model classes.
Instead, we can restrict the support of prior belief and posterior approximations to a feasible set, $\sFeasible = \left\{ \pQThCPsi \pQPsi \right\}$,
and \Cref{thm:parsimony_optimization} still provides a consistent framework to evaluate and compare the utility of such restrictions.}

The parsimony objective also allows us to understand and quantify memorization of training data in \Cref{cor:memorization} as a bound on the increase in model complexity that is required to achieve increased agreement between predictions and our training data.
\tAdd{This bound also holds when we restrict $\sFeasible$.}

\begin{corollary}
\label{cor:memorization} \thmTitle{Quantifying Memorization.}
We can write the combined model complexity terms as $\vChi[\pQThPsi] =  \KL{\pQThPsi}{\pP(\vTh, \vPsi)}$
and let $\pQsThPsi$ be the constrained optimizer of the parsimony objective, restricted to a given feasible set $\sFeasible = \left\{ \pQThPsi \right\}$.
Let the optimal predictions be written as
\submitBlank
\begin{align*}
\pQsY = \integ{d\vPsi\, d\vTh} \pP(\vY \mid \vTh) \pQsThPsi.
\end{align*}
Every feasible alternative $\pQThPsi$ must satisfy
\submitBlank
\begin{align*}
\vChi[\pQThPsi] - \vChi[\pQsThPsi] \geq \expect_{\pQThPsi} \info{\pRY}{\pP(\vY \mid \vTh)}{\pQsY},
\end{align*}
showing that any increased agreement with training data can only be achieved by a still greater increase in model complexity. 
\end{corollary}

\tAdd{\Cref{sub:polynomial_regression} includes a visualization of this complexity tradeoff, \Cref{fig:polynomial_memorization} corresponding to potential polynomial representations for a regression model.
As our experiments demonstrate, restricting our attention to classes of simple descriptions provides a tractable means to discover models and control complexity.}


\section{Implementation}
\label{sec:implementation}

The parsimony objective acts on opportunities for compression to reduce the complexity of our belief over models through both the description of prior belief and the information gained due to inference.
While there are many ways to encode the concepts that we need to articulate prior belief, compression is only possible if the interpreter admits a range of code lengths.
Consequently, it is important to review some efficient encodings, capable of expressing increasing degrees of specificity with longer codes, that are needed by our prototype implementations.
Then we discuss our algorithms for polynomial regression followed by decision trees.

\subsection{Useful Encodings}
\label{sub:encodings}

Sometimes we need to identify one of multiple states without any principle that would allow us to break the symmetry among potential outcomes.
For example, our decision tree algorithm requires a feature dimension to be specified from $n$ possibilities.
Laplace's principle of insufficient reason indicates that our encoding should not break symmetry among hypothetical permutations of the features.
We can easily handle this case by representing each state with a single symbol from an alphabet of $n$ possibilities.

As the cardinality of the set increases, however, the information provided by realizing a symbol increases logarithmically.
Thus, this approach cannot hold when we have countably infinite sets, such as with integers or rational numbers, or information would diverge.
Instead, we must break symmetry with either some notion of magnitude, some notion of precision, or both.

\subsubsection{Nonnegative integers}

Rissanen's universal prior over integers \citep{Rissanen1983} can be derived by counting outcomes over binary sequences of increasing length.
Provided the sequence length is known, any nonnegative integer $z$ can be encoded with $\lfloor \log_2(z+1) \rfloor$ binary digits,
as shown in \Cref{fig:rissanen_1}.
Yet, the sequence length is also a nonnegative integer, thus a recursive encoding of arbitrary nonnegative integers will have length approaching 
\submitBlank
\begin{align*}
\log_2^*(z) =& \lfloor \log_2(z+1) \rfloor + \lfloor \log_2\left( \lfloor \log_2(z+1) \rfloor + 1 \right) \rfloor\\
& + \lfloor \log_2\left( \lfloor \log_2\left(  \lfloor \log_2(z+1) \rfloor +1 \right) \rfloor + 1\right) \rfloor +\cdots.
\end{align*}

\begin{table}[h]
\centering
\begin{tabular}{| r | r | r | r | r | r | r | r | r | r | r |}
\thickhline 
Sequence &  & 0 & 1 & 00 & 01 & 10 & 11 & 000 & 001 & $\cdots$ \\ \thickhline
$z$ & 0 & 1 & 2 & 3 & 4 & 5 & 6 & 7 & 8 & $\cdots$ \\ \thickhline
\end{tabular}
\caption{\small
Enumeration of binary sequences of increasing length.
}
\label{fig:rissanen_1}
\end{table}

\uAdd{The Elias $\gamma$ coding \citep{Elias1975} simply represents the sequence length with a negated unary prefix.
Elias $\delta$ codes add one recursion, thus representing the sequence length with a $\gamma$ code.
Elias $\omega$ codes allow for arbitrary recursions by building up positive integers that either represent the length of the next sequence or the final outcome.
Decoding begins with the initial value $N=1$.
If the next bit is $0$, then $N$ is the final value.
Otherwise, the leading $1$ followed by $N$ bits encodes the updated value of $N$.
The process repeats until the next segment has a leading $0$.
We can take $z=N-1$ for nonnegative integers.}

\uAdd{In practice, however, we do not need representations for arbitrarily large integers and we can obtain more efficient codes with a limited maximum representation.
For example, we can start with a single bit to represent the length of the subsequent code segment and iterate representations from \Cref{fig:rissanen_1} a predetermined number of times.
If we know how many recursions are needed to represent a maximum integer, we obtain a code that approximates Rissanen's universal prior.}
\Cref{fig:rissanen_2} shows how the first few Rissanen codes are formed.
This encoding becomes very efficient for large integers, but the number of length recursions must be set high enough.

\begin{table}[h]
\centering
\begin{tabular}{| r | r | r | r | r | r | r | r | r || r | r || r | r || r | r || r |}
\thickhline 
$\vPsi_0$ & 0 & 1 & 1 & 1   & 1   & 1   & 1 \\ \thickhline
$z_0$        & 0 & 1 & 1 & 1   & 1   & 1   & 1 \\ \thickhline
$\vPsi_1$ &    & 0 & 0 & 1   & 1   & 1   & 1  \\ \thickhline
$z_1$        & 0 & 1 & 1 & 2   & 2   & 2   & 2 \\ \thickhline
$\vPsi_2$ &    & 0 & 1 & 00 & 01& 10 & 11 \\ \thickhline
$z_2$    & 0 & 1 & 2 & 3   & 4   & 5   & 6  \\ \thickhline
\end{tabular}
\caption{\small
The first sequence $\vPsi_0$ has an implied length of 1 bit.
The represented outcome $z_0$ indicates the length of $\vPsi_1$ and so on.
$\text{Rissanen}_i$ codes are formed by concatenation $(\vPsi_0, \vPsi_1, \ldots, \vPsi_i)$.
With three length recursions, numbers 0 through 126 are compressed to use between 1 and 9 bits.
}
\label{fig:rissanen_2}
\end{table}

\uAdd{We can also obtain good compression by using a single symbol to indicate the length of the remaining sequence.
Length-symbol codes also approximate the scaling invariance of Jeffrey's prior.
See \Cref{sec:ardprior} for further discussion of this property.
\Cref{fig:unary} provides a comparison.
Although we show codes with a 2-bit length symbol for easy comparison to other binary codes, the length symbol does not need to have a binary representation in general.}

\DeclareDocumentCommand \norep{}{n.r.}
\begin{table}[h]
\centering\uAdd{
\begin{tabular}{| r | r | r | r | r | r | r | r | r |}
\thickhline 
$z$ & 0 & 1 & 2 & 3 & 4 & 5 & 6 & 7 \\ \thickhline
$\text{Unary}$ & 0 & 10 & 110 & 1110 & 11110 & 111110 & 1111110 & 11111110 \\ \thickhline
$\text{Elias } \gamma$ & 1 & 010 & 011 & 00100 & 00101 & 00110 & 00111 & 0001000 \\ \thickhline
$\text{Elias } \delta$ & 1 & 0100 & 0101 & 01100 & 01101 & 01110 & 01111 & 00100000 \\ \thickhline
$\text{Elias } \omega$ & 0 & 100 & 110 & 101000 & 101010 & 101100 & 101110 & 1110000 \\ \thickhline
$\text{Rissanen}_1$ & 0 &  10 & 11 & \norep & \norep & \norep & \norep & \norep \\ \thickhline
$\text{Rissanen}_2$ & 0 &  100 & 101 & 1100 & 1101 & 1110 & 1111 &\norep \\ \thickhline
$\text{Rissanen}_3$ & 0 &  1000 & 1001 & 10100 & 10101 & 10110 & 10111 & 1100000 \\ \thickhline
$\text{Length-symbol}$ & 00 & 010 & 011 & 1000 & 1001 &1010 & 1011 & 11000 \\ \thickhline
\end{tabular}
\caption{\small
Nonnegative integers with unary codes, Elias codes, Rissanen codes, and a 2-bit length-symbol code.
Integers that have no representation are indicated by \norep
}}
\label{fig:unary}
\end{table}

\subsubsection{Binary Fractions}
\label{sub:binary_fractions}

It will also be useful to represent a dense distribution of fractions on the open unit interval $\vQ \in (0,1)$, thus allowing us to approximate any real number to arbitrary precision by a variety of potential transformations.
Binary fractions, with a denominator that is an integer power of 2 and a numerator that is odd, provide such a set with a convenient encoding.
These fractions can be written as
\submitBlank
\begin{align*}
q = \frac{2 i - 1}{2^{z+1}}
\quad\text{where}\quad
i \in [2^z]
\quad\text{and}\quad
z \in \mathbb{Z}_{\geq 0}.
\end{align*}

If we desire all fractions of a specific precision, corresponding to a fixed $z$, to have the same encoding length, then the numerator may be regarded as a single symbol with $2^z$ outcomes or $z$ bits.
We must also represent the precision $z$ with one of the integer encodings above.
\uAdd{As with the length-symbol codes, we can also represent $z$ with a single symbol that indicates the number of numerator bits to read.}
\Cref{fig:open_frac} shows some examples.

We can then translate and scale $q$ to represent an angle on the real Riemann circle, the corresponding real numbers are $r = \tan(\pi(q - 1/2))$,
but other choices are also possible, such as inverting the normal cumulative distribution function $r = \sqrt{2} \erf^{-1}(2q-1)$.
Multiplying the result by some $\sigma>0$ would set any desirable scale of outcomes.

\begin{table}[h]
\centering\uAdd{
\begin{tabular}{| r | r | r | r | r | r | r | r | r | r |}
\thickhline 
$q$                 & $1/2$ & $1/4$ & $3/4$ & $1/8$ & $3/8$   & $5/8$  & $7/8$  & $1/16$ & $\cdots$ \\ \thickhline
$\text{Code}$ & 00        & 010     & 011    & 1000 & 1001 & 1010 & 1011 & 11000 & $\cdots$ \\ \thickhline
\end{tabular}
\caption{\small
Leading binary fractions on the open unit interval with a 2-bit encoding of $z$.
}}
\label{fig:open_frac}
\end{table}

\subsection{Polynomial Regression}
\label{sub:polynomial_regression}

Our regression prototype directly encodes a polynomial with a description of coefficients and captures uncertainty with a hyperposterior ensemble.
Since any $n$th degree polynomial can be written as a linear combination of $n+1$ basis functions of ascending degree, we first need to identify the degree of polynomial.
\uAdd{Our experiments show that a length-symbol encoding serves this purpose well.}
Coefficients are represented in the Chebyshev basis.
Since critical points equioscillate in this basis, the corresponding polynomial coefficients are interpretable as the length scales of oscillation.
Because we expect all $n+1$ coefficients to take nontrivial values, the encoding reserves a variable-length segment for each coefficient, rather than attempting to use a sparse encoding.
Still, natural sparsity will result from binary fractional codes, representing angles on the Riemann circle, after transforming them to the corresponding real numbers.

\begin{algorithm}[h!]
\caption{Parsimonious Polynomial Regression Gibb's Sampler}
\label{alg:pars_polyreg}
\fontsize{10}{16}\selectfont
\begin{algorithmic}[1]
\Require Vectors $\vX$ and $\vY$ provide abscissas and ordinates, respectively,
with $\vY$ scaled so that the intrinsic stochasticity of the process is $\sigma=1$.
\tAdd{Generate $n$ samples from polynomials using at most $b$ Chebyshev basis functions.}
\Ensure $\mPsi = \{\vPsi\}$ is an ensemble of hyperposterior polynomial descriptions $\vPsi_i \sim \pP(\vPsi \mid \mX, \vY)$.
\Function{ParsimoniousRegression}{$\vX, \vY, n, b$}
\State Initialize $\vPsi$ to the zero polynomial
\For{each sample iteration $i = 1, 2, \dots, n$}
\State Generate a random permutation of the basis functions.
\For{each permuted coefficient $j = 1, 2, \dots, b$}
\State Identify the leading nonzero coefficient $k$ in $\vPsi$.
\State Form tensor product of all representable perturbations over both $j$ and $k$.
\State Update $\vPsi$ by sampling the hyperposterior, restricted to these perturbations.
\EndFor
\State \tAdd{Add $\vPsi$ to the hyperposterior ensemble $\mPsi$.}
\EndFor
\EndFunction
\end{algorithmic}
\end{algorithm}

\Cref{alg:pars_polyreg} samples the hyperposterior using a nonreversible sequence of reversible samples over all representable polynomial coefficients.
\tAdd{
We regard each sample of polynomial coefficients as a Dirac delta distribution concentrated at a single polynomial.
This is equivalent to treating the complexity hyperprior as an ordinary prior over polynomial descriptions.
Because our encoding length is most sensitive to the degree, our sampler proposes all joint perturbations of the leading nonzero and each other coefficient in a randomly permuted order.}
All other coefficients are held fixed in each proposal set.
For each coefficient, the sampler considers all binary fractions with $z \leq 4$.

\begin{figure}[h!]
	\centering
	\includegraphics[width=0.825\textwidth]{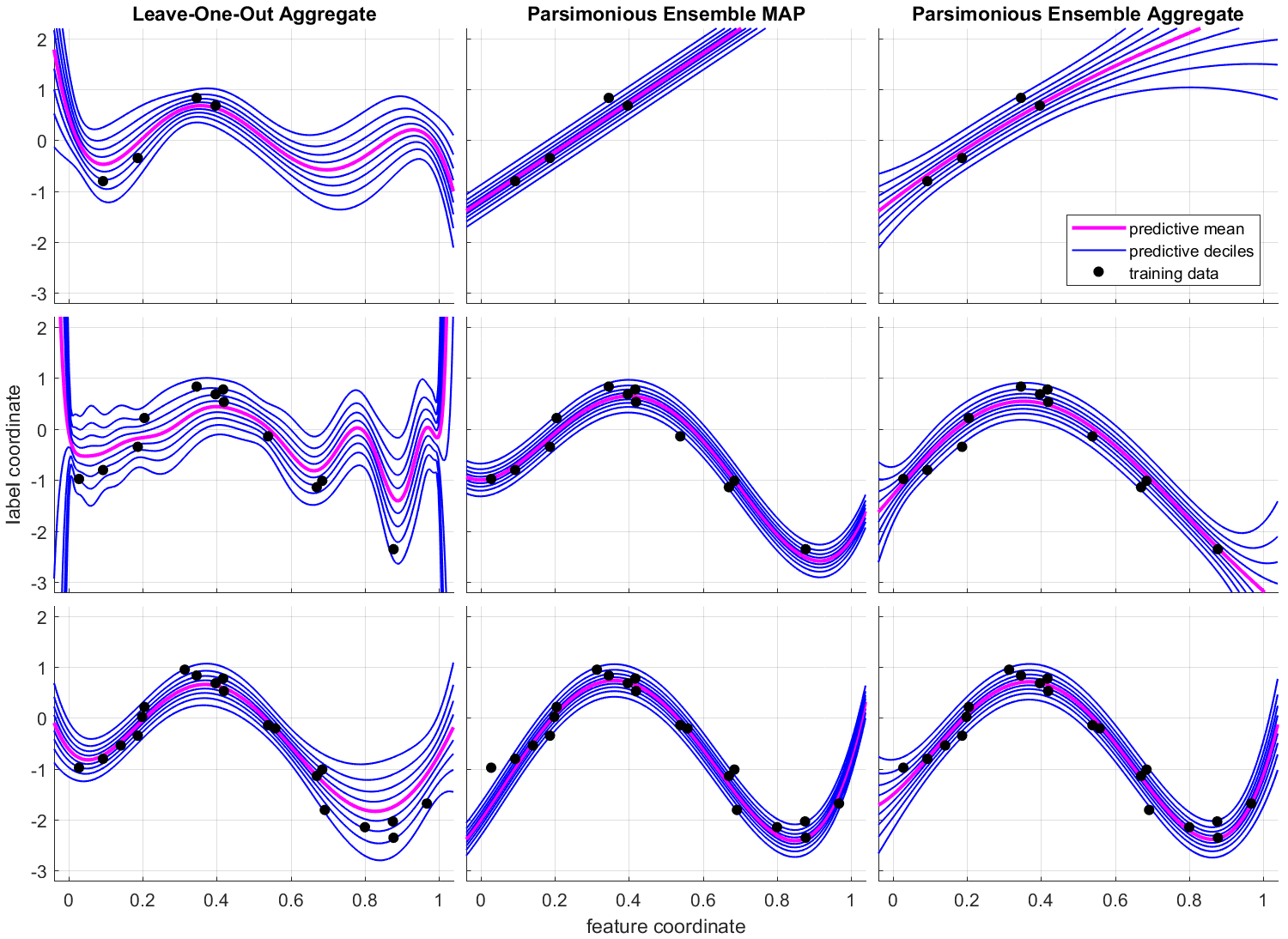}
	\caption{\vspace{-0.1in}\small
	Regression experiments comparing leave-one-out cross validation (left column), the hyper-MAP (middle column), and the hyperposterior aggregated over 50 samples (right column).
	All data come from the same ground truth.
	The hyper-MAP is consistently simpler than the leave-one-out aggregate.
	The hyperposterior aggregate naturally captures extrapolation risk, increasing uncertainty as we deviate from data.
	More data allow modest increases in complexity to reduce uncertainty.}
  \label{fig:regression}
\end{figure}

\tAdd{\Cref{fig:regression} compares the complexity suppression of leave-one-out cross-validation with our results from \Cref{alg:pars_polyreg} using 20th degree polynomials.
In order to provide a fair comparison, we generate 21 full leave-one-out ensembles, corresponding to each polynomial degree, and then select the ensemble with the best average over holdout log-likelihoods.
Thus, both approaches explore the same model families, wherein all coefficients above a certain degree are zero.
We see that, if the hyperparameter search covers a range of dimensions, the standard approach achieves some, albeit limited, success in identifying relatively low complexity ensembles.
}\uAdd{We also tested unary, Elias $\gamma$, and sufficient Rissanen codes for both the polynomial degree and corresponding binary fractions using the data in the second row of \Cref{fig:regression}.
Specifically, we used $\text{Rissanen}_3$ codes for the polynomial degree and $\text{Rissanen}_2$ for the binary fraction precision.
The resulting aggregate parsimony objectives are compared in \Cref{fig:reg_codes}.
}

\begin{table}[h]
\centering\uAdd{
\begin{tabular}{| r | r | r | r | r |}
\thickhline 
Polynomial encoding & Unary & Elias $\gamma$ & Rissanen & Length-symbol \\ \thickhline
Parsimony objective (bits) & $103.0$ & $104.5$ & $101.9$ & $104.8$ \\ \thickhline
\end{tabular}
\caption{\small
Length-symbol codes give the optimal parsimony objective for this dataset.
}}
\label{fig:reg_codes}
\end{table}

\begin{figure}[h!]
	\centering
	\includegraphics[width=0.825\textwidth]{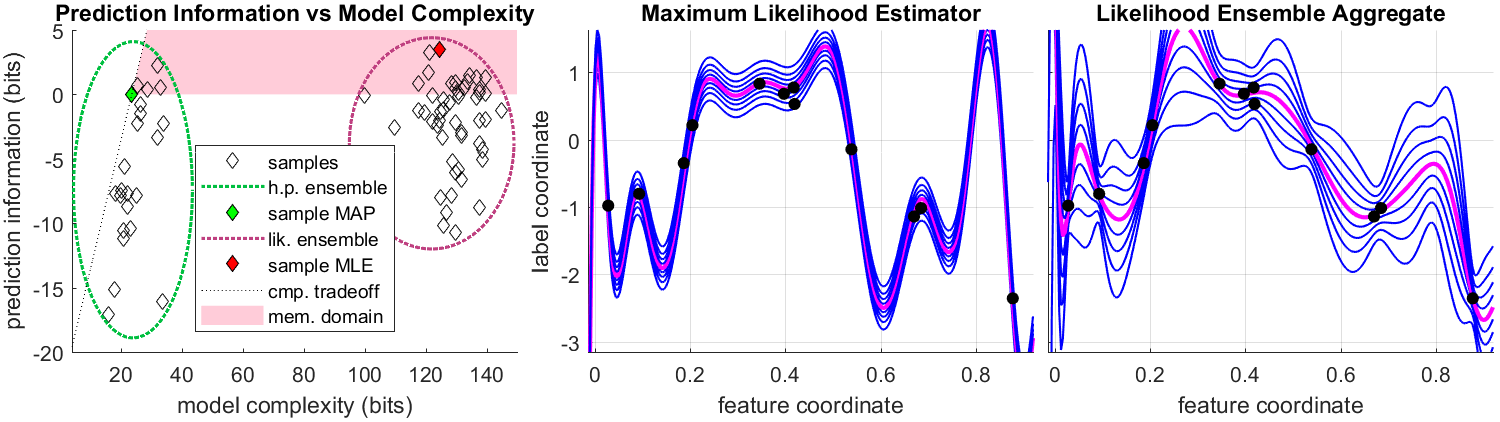}
	\caption{\vspace{-0.1in}\small
	Memorization demonstration with regression.
	We plot prediction information against model complexity (left) for models obtained from \Cref{alg:pars_polyreg} (left cluster)
	and likelihood samples from the same algorithm by disregarding generalized length (right cluster).
	Likelihood samples have much longer descriptions, some achieving better agreement with the data (memorization domain).
	Neither the MLE (middle) nor the likelihood ensemble (right) perform well.
	\vspace{-0.1in}
	}
  \label{fig:polynomial_memorization}
\end{figure}

\tAdd{\Cref{fig:polynomial_memorization} demonstrates memorization by applying \Cref{alg:pars_polyreg} to the same data in the second row of \Cref{fig:regression}, but replacing the hyperprior with a uniform distribution,
effectively sampling the likelihood by disregarding generalized length.
Plotting prediction information against model complexity (left) for both ensembles shows that the likelihood ensemble has much higher complexity.
If we constrain feasible beliefs to only single samples from either of the ensembles,
then the MAP defines the complexity tradeoff limit in \Cref{cor:memorization}, as well as the corresponding memorization domain in which
models may gain increased agreement with the data at the cost of an even greater increase in complexity.
Memorization is worst at the Maximum Likelihood Estimator (MLE), found within the likelihood ensemble.
The corresponding predictions (middle) closely fit the data.
Predictions from the likelihood ensemble (right) show increased uncertainty, but still violate Occam's Razor,
thus underscoring the central role of generalized length in suppressing memorization.}

\subsection{Decision Trees}
\label{sub:decision_trees}
\DeclareDocumentCommand \nLabel{} {\ell}
\DeclareDocumentCommand \nFeature{} {k}
\DeclareDocumentCommand \iLabel{} {y}
Decision trees predict discrete classifications, labels, by evaluating a sequence of binary decisions.
Each case in our training dataset is represented by both a feature vector $\vX$, with $\nFeature$ components that are each comparable to a threshold, and an enumerated label $y \in [\nLabel]$, where $\nLabel$ is the number of labels.
Evaluation begins at the root note, representing the axis-aligned bounding box of potential features.
The node specifies a feature dimension and a comparison threshold serving to partition the feature domain into two components, the left and right child nodes.
The comparison outcome indicates membership and the process iterates so that a sequence of binary decisions filters each case through a series of increasingly restrictive partitions,
each of which is intended to simplify the classification problem.
This filtration terminates at a leaf node that specifies either a single label or, more generally, probabilities over all labels.

Decision trees are trained using a recursive process that also begins at the root node.
We take the set of training cases that are members of a given node and we must either construct a branch structure or halt splitting and finalize label probabilities.
The conventional procedure evaluates every potential splitting outcome with some utility function and then chooses the optimizer.
While a wide variety of utility functions are used in practice,
a standard information-theoretic approach maximizes the reduction in entropy due to the splitting.

Let $\vC_{\iLabel}$ represent the count of training cases with the label $y$ that fall within a given node domain.
If the node were a leaf, the frequentist approach to predicting label probabilities would use the sample mean  
\submitBlank
\begin{align*}
\vMu_{\iLabel} = \frac{\vC_{\iLabel}}{c}
\quad\text{where}\quad
c = \sum_{\iLabel=1}^{\nLabel} \vC_{\iLabel}. 
\end{align*}
We denote the corresponding variables for hypothetical left and right child nodes using superscripts, e.g.~$\vC_{\iLabel}^{(L)}$ and $\vC_{\iLabel}^{(R)}$, respectively.
The reduction in entropy associated with a potential splitting, weighted by the fraction of cases that appear within the respective domains of each child, is
\submitBlank
\begin{align*}
& \Delta S =
\frac{c^{(L)}}{c} \sum_{\iLabel=1}^{\nLabel} \vMu_{\iLabel}^{(L)} \log_2( \vMu_{\iLabel}^{(L)} ) + \frac{c^{(R)}}{c} \sum_{\iLabel=1}^{\nLabel} \vMu_{\iLabel}^{(R)} \log_2( \vMu_{\iLabel}^{(R)} ) - \sum_{\iLabel=1}^{\nLabel} \vMu_{\iLabel} \log_2( \vMu_{\iLabel} ). 
\end{align*}
The splitting that maximizes $\Delta S$ is accepted and the resulting child nodes are trained recursively until no further reduction is possible, i.e.~when a node contains only cases of a single label.

Bootstrap aggregation constructs an ensemble of decision trees, a random forest, by resampling the dataset.
This consists of forming a new dataset, the same size as the original, by sampling the original dataset uniformly with replacement.
Predictions are then aggregated by taking the average over the ensemble.

\tAdd{Our approach encodes each decision tree as a state of prior belief on the corresponding partition of feature coordinates.
We describe the partition recursively using an encoding for each node in the binary search tree.
Each node begins by specifying whether it is a leaf or branch with 1 bit.
If the node is a branch, a symbol from $\nFeature$ possibilities gives the feature dimension used in the comparison, thus contributing $\log_2(\nFeature)$ bits to the generalized length.
Since we can translate and scale features in the given dimension to the unit interval $[0,1]$,
we can represent any splitting threshold as a binary fraction on the open unit interval.
Note that it is never useful to split at either $0$ or $1$.
The encoding continues by describing the left and right children.}

\tAdd{If the node is a leaf, its feature partition has an independent prior, a flat Dirichlet distribution, over the simplex of all coherent label probabilities.}
Let $\vTh$ represent a vector of label probabilities within a single leaf.
We can perform exact inference using label counts $\vC_{\iLabel}$ to recover Laplace's rule of succession
\submitBlank
\begin{align*}
\expect_{\pP(\vTh \mid \vYr)} \left[ \vTh_{\iLabel} \right] = \frac{\vC_{\iLabel} + 1}{c + \nLabel},
\end{align*}
which is the posterior-predictive distribution for labels of new data that land within this leaf.

\tAdd{Our parsimonious decision trees are generated using a recursive process that begins by calling \Cref{alg:pars_node} on the full training dataset to form the root node.
The remaining structure is generated by sampling both a feature dimension and threshold or by halting to form a leaf node that infers label probabilities as above.
If the structure splits, then data are partitioned accordingly, and the process repeats on the left and right children.}

\begin{algorithm}
\caption{Parsimonious Node Construction for Decision Trees}
\label{alg:pars_node}
\fontsize{10}{16}\selectfont
\tAdd{
\begin{algorithmic}[1]
\Require $\mX$ is a matrix of features and $\vY$ is the corresponding vector of labels for data in this node's partition.
\tAdd{We sample an approximate annealed hyperposterior, controlled by an annealing schedule $\vAlpha$ and recursion depth $d$.}
\Ensure Sample a node description $\vPsi$ with proposal probability $\pS(\vPsi)$ and unnormalized hyperposterior $\pP(\vPsi \mid \mX, \vY)$.
\Function{$\left[\vPsi,\, \pS(\vPsi),\, \pP(\vPsi \mid \mX, \vY) \right]$ = ParsimonyNode}{$\mX, \vY, \vAlpha, d$}
\State Let $\vPsi_0$ describe this node as a leaf.
\State Compute the likelihood $\pP(\vY \mid \vPsi_0)$ from a flat Dirichlet prior over all label probabilities.
\State Enumerate every possible feature domain splitting for this node as $i \in [n]$.
\For{\label{line:pars_node_for} each splitting $i = 1, 2, \dots, n$}
	\State Let $\vPsi_i$ describe this node as a branch. 
	\State Partion data with the splitting, $(\mX^{(L)}, \vY^{(L)})$ and $(\mX^{(R)}, \vY^{(R)})$.
	\State \label{line:apprx_post} Approximate the likelihood assuming both children are independent leaf nodes
\begin{align*}
	\pP(\vY \mid \vPsi_i) \approx \pP(\vY^{(L)} \mid \vPsi_i)\pP(\vY^{(R)} \mid \vPsi_i).
\end{align*}
\EndFor
\State Sample $\vPsir$ from $\pS(\vPsi_i) \propto \pP(\vY \mid \vPsi_i)^{\vAlpha_d} \pP(\vPsi_i)$ over $i = 0, 1, \ldots, n$.
\If{$\vPsir$ is a branch node}
	\State Partion data accordingly as $(\mX^{(L)}, \vY^{(L)})$ and $(\mX^{(R)}, \vY^{(R)})$.
	\State Recursively construct left and right child nodes as
\begin{align*}
	\left[\vPsi^{(L)},\, \pS(\vPsi^{(L)}),\, \pP(\vPsi^{(L)} \mid \mX^{(L)}, \vY^{(L)}) \right] & = \textbf{ParsimonyNode}(\mX^{(L)}, \vY^{(L)}, \vAlpha, d+1)\quad\text{and}\\
	\left[\vPsi^{(R)},\, \pS(\vPsi^{(R)}),\, \pP(\vPsi^{(R)} \mid \mX^{(R)}, \vY^{(R)}) \right] & = \textbf{ParsimonyNode}(\mX^{(R)}, \vY^{(R)}, \vAlpha, d+1).
\end{align*}
	\State Concatenate descriptions, $\vPsi = \left[ \vPsir,\; \vPsi^{(L)}, \vPsi^{(R)} \right]$.
	\State Compose sample probabilities, $\pS(\vPsi) = \pS(\vPsir) \pS(\vPsi^{(L)})\pS(\vPsi^{(R)})$. 
	\State Compose the hyperposterior
\begin{align*}
	\pP(\vPsi \mid \mX, \vY)  = \pP(\vPsi^{(L)} \mid \mX^{(L)}, \vY^{(L)})  \pP(\vPsi^{(R)} \mid \mX^{(R)}, \vY^{(R)})  \pP(\vPsir).
\end{align*}
\Else{\;($\vPsir$ is the leaf node)}
	\State Set $\vPsi = \vPsir$ and $\pS(\vPsi) = \pS(\vPsir)$.  
	\State Compute the unnormalized hyperposterior, $\pP(\vPsi \mid \mX, \vY)  = \pP(\vY \mid \vPsir) \pP(\vPsir)$.
\EndIf
\EndFunction
\end{algorithmic}
}
\end{algorithm}

\tAdd{If we wanted to sample each splitting from the exact posterior, we would need to marginalize over all elaborations to the node structure that could follow.
Since the number of such structures grows exponentially with depth, this would not be computationally feasible.
Instead, we must sample an approximate posterior by assuming children will be leaf nodes, as in the standard approach.
Unfortunately, this approximation can generate over-attractive splitting domains, thus making it difficult to sample high posterior alternatives.
To mitigate this problem, we anneal the likelihood in the posterior approximation to increase sample diversity.
In our experiments, we used an annealing schedule that disregards the likelihood for the first two branch depths, $\vAlpha = [0,\, 0,\, 1,\, 1,\,\ldots]$.
The hyperposterior predictive integral can then be corrected using importance weighting over the ensemble of decision trees.
For each tree $t$ with a description $\vPsi_t$, we need the probabilities of both the composition of sampled proposals $\pS(\vPsi_t)$ and the posterior $\pP(\vPsi_t \mid \mX, \vY)$, up to the unknown normalization.
This gives importance weights}
\submitBlank
\begin{align*}
w_t \propto \frac{\pP(\vPsi_t \mid \vYr)}{\pS(\vPsi_t)}
\quad\text{so that}\quad
\sum_t w_t = 1.
\end{align*}

Because we are performing exact inference, we could skip the following information analysis and compute the log likelihood directly from label counts
\submitBlank
\begin{align*}
\log_2\left( \pP(\vYr \mid \vPsi) \right) = \log_2\!\left(\frac{\Gamma(\nLabel)}{\Gamma(c+\nLabel)}\right) +\sum_{\iLabel=1}^{\nLabel} \log_2(\Gamma(\vC_{\iLabel}+1)).
\end{align*}
However, we can use this opportunity to demonstrate how the analysis would proceed if a variational approximation were used
\tAdd{by simply replacing $\pP(\vTh \mid \vYr)$ below with any distribution from a feasible set, $\pQTh \in \sFeasible$.}

The amount of information due to change in model belief is
\submitBlank
\begin{align*}
&\KL{\pP(\vTh \mid \vYr)}{\pP(\vTh)} \\
&= \frac{1}{\log(2)} \left(
\log\!\left(\frac{\Gamma(c+\nLabel)}{\Gamma(\nLabel)}\right) - c \digamma(c + \nLabel)
+ \sum_{\iLabel=1}^{\nLabel} \vC_{\iLabel} \digamma(\vC_{\iLabel} + 1) - \log(\Gamma(\vC_{\iLabel} + 1)) \right)
\end{align*}
where $\digamma(x) = \frac{d}{dx} \log(\Gamma(x))$ is the digamma function.
The prediction information gained about labels from a uniform prior is
\submitBlank
\begin{align*}
\expect_{\pP(\vTh \mid \vYr)} \info{\pRY}{\pP(\vY \mid \vTh)}{\pP(\vY \mid \vTh_0)}
= \frac{1}{\log(2)} \left( c \log(\nLabel) - c \digamma(c+\nLabel) + \sum_{\iLabel=1}^{\nLabel} \vC_{\iLabel} \digamma(\vC_{\iLabel}+1) \right).
\end{align*}
Subtracting inference information from predictive information recovers the log likelihood up to an additive constant.

\begin{figure}[b!]
	\centering
	\includegraphics[width=0.75\textwidth]{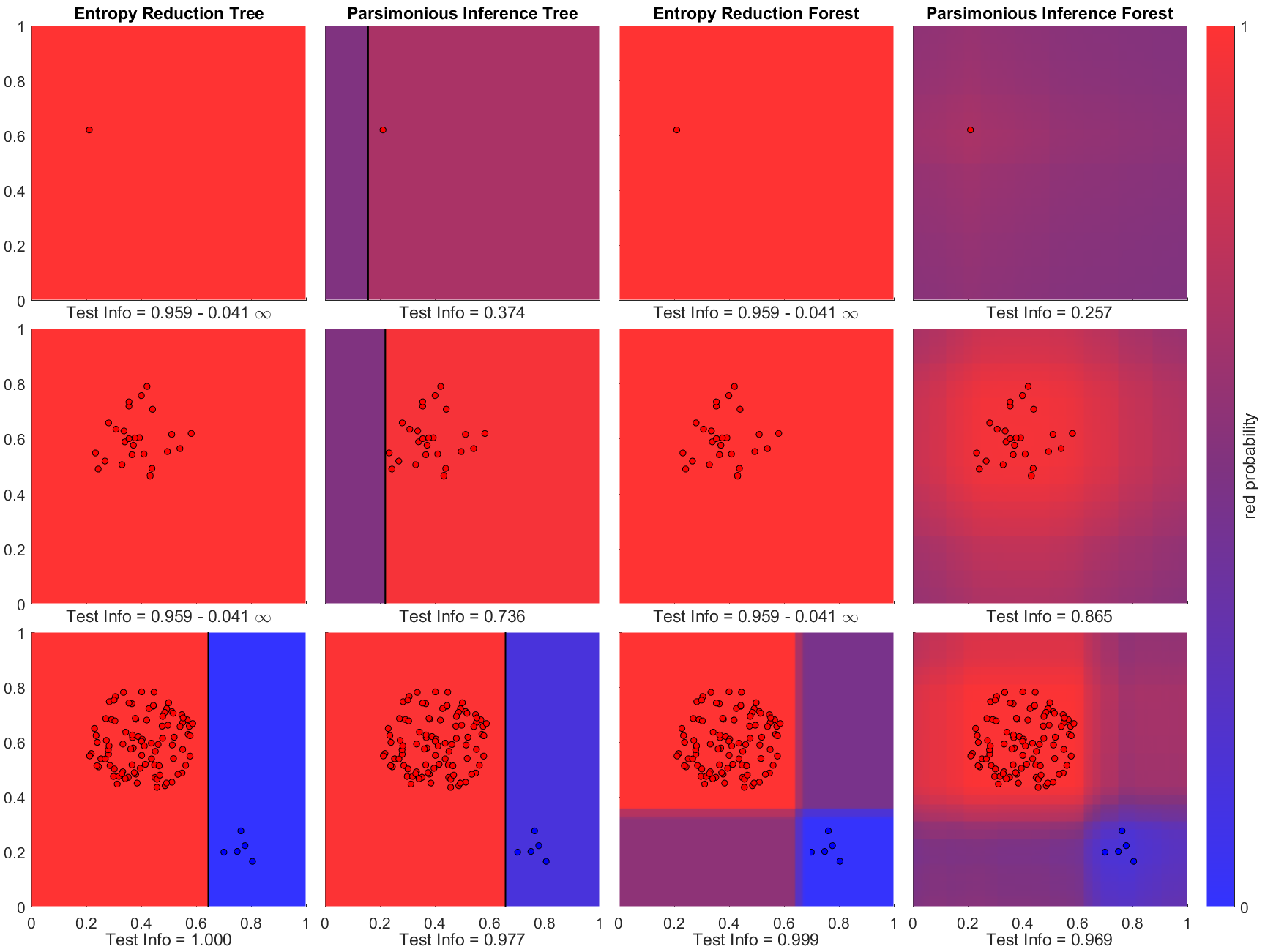}
	\vspace{-0.1in}
	\caption{\small
	 These first decision tree experiments use a highly skewed generative process with well-separated label domains.
	Conventional decision trees and random forests, columns 1 and 3, obtain highly confident predictions despite having few data,
	Our parsimonious trees and hyperposterior aggregates, columns 2 and 4, only gradually reduce uncertainty and
	demonstrate better extrapolation uncertainty.
	\vspace{-0.1in}
	}
  \label{fig:decision_tree_1}
\end{figure}

Our first set of decision tree experiments, \Cref{fig:decision_tree_1}, examines learning from a generative process that cleanly partitions the data into regions containing a single label.
Row one compares models that learn from a single sample.
Rows two and three learn from 25 and 100 samples, respectively.
The leading two columns compare a conventional decision trees with parsimonious decision trees.
The last two columns compare bootstrap aggregation with our hyperposterior aggregates.
\tAdd{Every aggregate ensemble contains 1000 decision tree samples.}

This is a highly skewed generative process;
even with 25 samples, the second row still has no realizations of a blue label.
Yet, the parsimonious aggregate predictions in both rows 1 and 2 naturally increase uncertainty as the prediction domain deviates from the training data.
In contrast, the conventional approach obtains absolute certainty in regions that lack data.
Blue labels finally appear in the third row, showing how the parsimonious forest reacts to skewed data.
This generative process is incapable of generating data in the off-diagonal regions,
but without any way of knowing that, we should be highly skeptical of certainty in the absence of evidence.

\begin{figure}[b!]
 	\centering
	\includegraphics[width=0.75\textwidth]{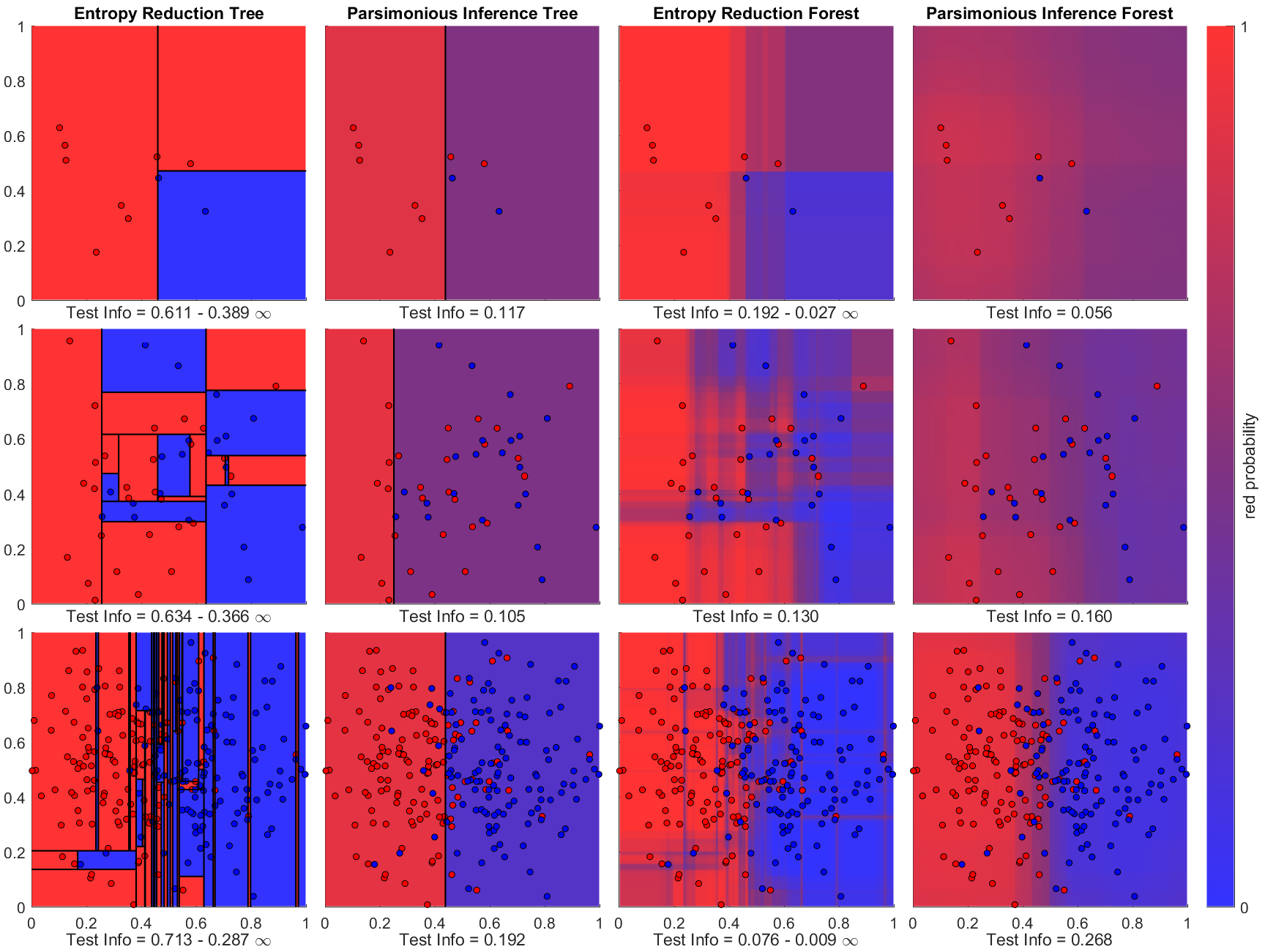}
	\vspace{-0.1in}
	\caption{\small
	Our second decision tree experiments investigate a generative process with smooth mixing of labels.
	Conventional approaches increase complexity rapidly with more data, yielding predictive artifacts that hew to few points. 
	In contrast, parsimonious trees remain simple, only gaining confidence with sufficient evidence.
	\vspace{-0.1in}
	}
  \label{fig:decision_tree_2}
\end{figure}

The second set of experiments, \Cref{fig:decision_tree_2}, examines a generative process that mixes labels.
There is nonzero probability of generating a point of either label at any location, but red labels are more likely to appear on the left and blue on the right.
We observe that typical decision trees are more complicated than their parsimonious counterparts, as expected.
Moreover, complexity increases rapidly as our dataset grows.
In contrast, parsimonious decision trees increase complexity gradually.
The typical approach also yields confident artifacts in the vicinity of few data,
whereas parsimonious trees and parsimonious forests only gradually reduce uncertainty.


\section{Discussion}
\label{sec:discussion}

Because our formulation of complexity accounts for information from arbitrary descriptions,
it already contains the functionality needed to address a wide variety of challenges.
First, we examine how to include other sources of prior belief in this framework and how to extend it to prior belief over multiple interpreters.
Second, we discuss how changing the scope of descriptions to account for symbols generated and communicated by elementary operations during the evaluation of predictions provides a mechanism to prefer fast algorithms.
Third, we explore how other Bayesian hyperpriors relate to description complexity.
Fourth, we compare the non-Bayesian treatment of probability in Rissanen's Minimum Description Length to our approach.
Finally, we offer our concluding remarks.

\subsection{Comparing and Inferring Interpreters}
\label{sec:interpreters}

We can easily compare interpreters within the same theoretical framework by identifying a common \tAdd{\textit{language},
by which we indicate an interpreter that is also Turing-complete.
Let $\mPhi = \left\{ \vPhi_i \mid i \in [n]\right\}$ represent an ensemble of interpreters written in the common language $\vPhi^*$.
If $\vPhi_i$ is capable of producing a state $\vObj$ from an encoding $\vPsi_i(\vObj)$ and the interpreter itself is encoded by $\vPsi^*(\vPhi_i)$ within the common language,
then applying the same hyperprior gives
\begin{align*}
\pP(\vObj, \vPhi_i) = \pP(\vObj \mid \vPsi_i) \pP(\vPsi_i) = 2^{-\fLen(\vPsi_i(\vObj)) - \fLen(\vPsi^*(\vPhi_i))}. 
\end{align*}
This is equivalent to simply prepending each state encoding with that of the relevant interpreter, $\vPsi^*(\vObj, \vPhi_i) = \left[\vPsi^*(\vPhi_i),\, \vPsi_i(\vObj)\right]$,
and regarding the common language as the shared interpreter of all such codes.}

\tAdd{Although this approach merely shifts the burden of how we derive interpreter validity to another level of abstraction,
we can still obtain practical insight into credible interpreters with this view.
Nefarious interpreters, such as the third basis in \Cref{fig:sensitivity}, effectively transfer complexity from an otherwise long encoding, subject to a simple interpreter, to a short encoding, subject to a long interpreter. 
Yet, when we shift the derivation of plausibility to a common language, the excessive complexity becomes visible in both cases.}

This approach also provides a formulation for grammar discovery through inference.
If we have several datasets and associated learning problems that should be explainable within a common language, then we can infer the structure of an efficient interpreter.
An interpreter that represents common functions among the different learning problems efficiently will be more likely than one that solves a single problem well by hiding complex functions with shortcuts.

\tAdd{Ultimately, however, deriving interpreter validity from a common language alone creates a problem of infinite regress;
the question remains of how we may derive plausibility among several languages.
Although we hold that this question does not impede practical applications, it can be answered by appealing to consistency through simulation.
\Cref{thm:utm_ensemble} presents a unique universal prior, consistent with our theoretical framework, for any ensemble of languages.
The proof immediately follows.}

\begin{theorem}
\label{thm:utm_ensemble}
\tAdd{\thmTitle{Universality of Consistent Belief in Turing-Complete Ensembles.}
Given an ensemble of Turing-complete interpreters, $\mPhi = \left\{ \vPhi_i \mid i \in [n]\right\}$,
we may consider an arbitrary state of prior belief, $\pP(\vPhi_i)$ over $i \in [n]$,
and apply \Cref{cor:probability_from_length} to obtain the joint probability of one interpreter simulating another 
\submitBlank
\begin{align*}
\pP(\vPsi_i, \vPsi_j) = \pP(\vPsi_i \mid \vPsi_j) \pP(\vPsi_j)\quad\text{for all}\quad i,j \in [n].
\end{align*}
There exists a unique prior for which the marginalized simulation probability recovers the prior,
thus consistently accounting for simulation complexity within prior belief.}
\end{theorem}

\begin{proof}\textbf{of \Cref{thm:utm_ensemble}.}
\tAdd{A simulator $\vPsi_{ij}$ is a code sequence that may be prepended to any valid code for $\vPhi_i$ and allow it to run on $\vPhi_j$.
Let the set $\mPsi_{ij} = \left\{ \vPsi_{ij} \right\}$ include all simulators for language pairs indexed $i, j \in [n]$.
Applying \Cref{cor:probability_from_length} yields the hyperprior transition matrix, expressed elementwise by row $i$ and column $j$ as 
\submitBlank
\begin{align*}
\pP(\vPhi_i \mid \vPhi_j) = \left(\sum_k \sum_{\vPsi_{kj} \in \mPsi_{kj}} 2^{-\fLen(\vPsi_{kj})} \right)^{-1} \sum_{\vPsi_{ij} \in \mPsi_{ij}} 2^{-\fLen(\vPsi_{ij})},
\end{align*}
where each column is normalized over the languages in the ensemble.
A consistent prior must satisfy $\pP(\vPhi_i) = \sum_j \pP(\vPhi_i \mid \vPhi_j) \pP(\vPhi_j)$ for all $i \in [n]$.
Because all Turing-complete languages can simulate one another, $\mPsi_{ij}$ is nonempty for all pairs.
Existence and uniqueness immediately follow by applying the Perron--Frobenius Theorem;
every square matrix with positive entries has a unique largest eigenvalue and the paired eigenvector may be constructed to have positive entries.
Since the transition matrix maps every normalized state to another normalized state, that eigenvalue must be 1 and no other eigenvectors may be coherent probabilities.}
\end{proof}

\tAdd{This result shifts remaining subjectivity to the set of interpreters we are willing to consider for comparison.
In practice, the shortest simulator in each set, say $\vPsir_{ij} \in \mPsi_{ij}$, will dominate the corresponding matrix element.
Thus, a more practical approximation of the hyperprior transition matrix is
\begin{align*}
\pP(\vPhi_i \mid \vPhi_j) \approx \left(\sum_k 2^{-\fLen(\vPsir_{kj})}\right)^{-1} 2^{-\fLen(\vPsir_{ij})}.
\end{align*}
Note $2^{-\fLen(\vPsir_{ii})}=1$ for all $i \in [n]$, since the shortest self-simulator is trivial in each language, $\vPsir_{ii} = \emptyset$.
We find this result intuitive because it suppresses interpreters that require excessive complexity to simulate.
In the absence of such a computation, we may only conclude that we should prefer interpreters that appear to be simple.}

\subsection{Integrating Additional Beliefs}

Although this work is motivated by the difficulty of expressing prior belief over abstract models,
when we have access to additional information that could constrain prior beliefs, that information may be impactful.
Therefore, we should be able to integrate other prior beliefs within the general complexity framework.
Let our complexity-based prior belief be denoted as $\pP(\vObj \mid \mathcal{C}) = 2^{-\fLen(\vPsi(\vObj))}$.
If we also have other prior beliefs, $\pP(\vObj \mid \mathcal{B})$, we can form the composite prior $\pP(\vObj \mid \mathcal{B}, \mathcal{C})$.
For example, $\mathcal{B}$ may express physical laws or previously observed data.
One approach would be to use an interpreter that implicitly embeds $\mathcal{B}$ within viable encodings so that $\pP(\vObj \mid \mathcal{B}, \mathcal{C}) = 2^{-\fLen(\vPsi_{\mathcal{B}}(\vObj))}$.
Alternatively, if we assume that belief derived from $\mathcal{B}$ is conditionally independent of our complexity-based belief $\mathcal{C}$, then we have
\submitBlank
\begin{align*}
\pP(\vObj \mid \mathcal{B}, \mathcal{C}) &= \frac{\pP(\mathcal{B} \mid \vObj, \mathcal{C}) \pP(\vObj \mid \mathcal{C})}{\pP(\mathcal{B} \mid \mathcal{C})}
=\frac{\pP(\mathcal{B} \mid \vObj ) \pP(\vObj \mid \mathcal{C})}{\pP(\mathcal{B} \mid \mathcal{C})} \\
&=\frac{\pP(\vObj \mid \mathcal{B}) \pP(\mathcal{B})  \pP(\vObj \mid \mathcal{C})}{\pP(\mathcal{\vObj}) \pP(\mathcal{B} \mid \mathcal{C})}
\propto \pP(\vObj \mid \mathcal{B}) \pP(\vObj \mid \mathcal{C}),
\end{align*}
where $\pP(\mathcal{\vObj})$ must be a constant for all $\vObj$ since both $\mathcal{B}$ and $\mathcal{C}$ have been constructed to capture all our beliefs.
Thus, the composite prior is easily formed up to a constant of proportionality.

\subsection{The Imperative of Utility}
\label{sec:computability}

It is not useful to consider models that, in order to provide a substantial contribution to predictions, would require more evidence than we anticipate having.
Likewise, it is not useful to consider models that would require either more computational energy, communication capacity, or time to evaluate than we can afford.
Practical models must be discoverable, and predictions must be computable.
The Kolmogorov complexity is well-known to be uncomputable, thus raising a natural concern that generalizing prior belief to arbitrary descriptions only exacerbates the problem.
Yet, the primary purpose of \Cref{thm:parsimony_optimization} is to show how information theory allows us to restrict our attention to feasible manifolds of belief,
while simultaneously allowing us to compare outcomes from different choices of restriction.
Because long descriptions are already exponentially suppressed \`a priori, the information we generate by refusing to consider long descriptions becomes small as the descriptions we drop become long.

Even so, it is instructive to examine uncomputability more careful as it motivates future directions.
Suppose we had an oracle $\Omega$ that would determine whether or not a program $\vObj$ is capable of reproducing a mapping $(\vX, \vY)$ in a finite amount of time
\submitBlank
\begin{align*}
\Omega(\vPhi, \vObj, \vX, \vY) = \begin{cases}
 \text{true} & \vY = \vPhi(\vObj, \vX) \\
 \text{false} & \text{otherwise}.
\end{cases}
\end{align*}
The existence of such an oracle would allow us to determine the Kolmogorov complexity by brute force, generating and checking programs in order of increasing length until the oracle returns true.
Further, it would be a trivial matter to write another brute force subroutine to identify the first sequence $\vY$ with Kolmogorov complexity above an arbitrarily high limit $\fKol_{\vPhi}(\emptyset, \vY) > \tau$.
By setting $\tau$ to exceed the combined lengths of the oracle and brute force subroutines, we would have succeeded in writing a program that contradicts the Kolmogorov complexity.

The core problem with this thought experiment is the arbitrarily large amount of memory and elementary operations that would be required to run the program.
Disregarding the halting problem, the brute force search would need to generate full programs in memory, while only incurring the cost of encoding a counter.
We may conclude that problems associated with computability will be alleviated if we simply include memory operations,
every symbol generated or transmitted between slow and fast levels of memory, in the definition of model \textit{evaluation length}.
\citet{Lempel1976} present a related framework to measure sequence production complexity as the minimum number of steps required to build a sequence from a production process to construct a hierarchy of subsequences.
\citet{Speidel2008} provides additional discussion of recent work by \citet{Titchener1998}.

\tAdd{Speed priors \citep{Schmidhuber2002} and related work by \citet{Filan2016} develop these approaches to articulate prior beliefs that prefer efficient algorithms in the context of binary UTMs.}
In this view, prior belief becomes an expression for the degree of utility considering a model would contribute to obtaining feasible predictions.
Building on these approaches will allow us to restrict our attention to models that can be evaluated with limited resources.
For example, randomized algorithms such as Randomized QR with Column Pivoting (RQRCP) \citep{Duersch2020b} would gain plausibility by having reduced slow communication bottlenecks.
In order for machine learning to be capable of providing discoverable, computationally feasible, and useful models, we cannot avoid limiting our attention accordingly.

\subsection{Relationship to other Bayesian Methods}
\label{sec:ardprior}

Our hyperprior provides a principled foundation to derive results that are similar in function to several well-known methods for specific Bayesian inference problems.
Notable comparisons include sparsity inducing priors, like Automatic Relevance Determination, for regression problems with continuous coefficients.
The ARD prior is a hyperprior over parameters which is intended to identify critical parameters and drive remaining parameters towards zero.
The ARD prior is implemented as:
\submitBlank
\begin{align*}
\pP(\vTh) = \pN(\vTh \mid 0, \vSigma)
\end{align*}
where we need to specify $\pP(\vSigma)$.
In the original work introducing the ARD prior, $\pP(\vSigma)$ is a gamma distribution in the precision $\tau = \sigma^{-2}$ with a small shape parameter.
This closely corresponds to the improper Jeffrey's prior $\pP(\vSigma) \propto \frac{1}{\vSigma}$, often used in practice for unknown scalar covariances \uAdd{because it is scaling invariant}.
If we partition the potential values of $\vSigma$ into intervals $0 < a < b < \infty$, where $a$ and $b$ are any positive real numbers,
the cumulative probability diverges for values less than $a$ and values greater than $b$, thus dominating over the finite contribution within $[a,b]$.
It follows that sampling the Jeffrey's prior would yield outcomes either very close to zero or diverging towards infinity. 
Within this formulation, if $\vTh$ has little relevance to the likelihood, then probability is maximized when $\vTh$ and $\vSigma$ approach zero.
Otherwise, a large enough $\vSigma$ will be found to allow $\vTh$ to take moderate nonzero values with the Jeffrey's prior introducing only a slight penalty as $\vSigma$ increases.
Therefore, it can be interpreted as making a binary choice between very large or very small $\sigma$.
More generally, a gamma distribution allows, indeed requires, the relative probably of these two outcomes to be tuned.

If we uniformly discretize the possible values of $\vSigma$ as $\vSigma_i = i \vSigma_1$ and assign them probabilities according to $\pP(\vSigma_i) \propto \frac{1}{\vSigma_i}$
\uAdd{up to a maximum value $i \in [M]$,}
we can equate this prior with the complexity prior $\pP(\sigma_i) = 2^{-\fLen(\sigma_i)}$ to obtain
\submitBlank
\begin{align*}
\fLen(\vSigma_i) = \log_2(\vSigma_i) + c = \log_2(i) + \log_2(\vSigma_1) + c = \log_2(i) + \fLen(\vSigma_1),
\end{align*}
where the constant $c$ gives the normalization.
Since the complexity of $\vSigma_i$ increases logarithmically in $i$,
we see that Jeffrey's prior is the continuous limit of the number of bits required to express \uAdd{an integer multiple of $\vSigma_1$.
We may interpret the fixed offset, $\fLen(\vSigma_1)$, as the contribution of a single symbol that determines the number of bits to read.} 

While sparsity is a useful notion of complexity for many problems, it is not universal.
Sparsity either regards a continuous parameter as either complex (nonzero) or not complex (zero).
While sparsity-inducing priors, like ARD, can compel continuous parameters to zero if they do not provide enough benefit to predictions,
they have no affordance to suppress other forms of complexity.
For example, there is no compelling notion of sparsity within the construction of decision trees.
Moreover, when we need to encode constants within prior descriptions, our theory supports consistent distinctions in complexity among potential constants.

\subsection{Relationship to Minimum Description Length}
\label{sub:mdl}

Rissanen's Minimum Description Length (MDL) shares many similarities with our theory, but it is not motivated by the philosophical foundations of reason that drive the Bayesian paradigm.
Rather, MDL views inference as finding an optimal compressed representation of a dataset and probability as a way of developing efficient codes.
MDL representations contain both the model used to construct an efficient code and the compressed form of the data that follows.
The length of the data representation $\fLen(\sData)$ is the sum of the number of bits needed to describe the model $\fLen(\vObj)$ and the number of bits needed to describe the residual data $\fLen(\sData \mid \vObj)$.
In its simplest form, the inference problem for identifying a hypothesis or program $\vObj \in \sObj$ is
\submitBlank
\begin{align*}
\fLen( \sData) = \min_{\vObj \in \sObj} \;\; \fLen(\vObj ) + \fLen(\sData \mid \vObj),
\end{align*}
requiring a specific discretization and encoding for a hypothesis space $\sObj$.
To address the arbitrary task of designing a hypothesis space encoding, MDL proposes a minimax optimization over universal codes,
minimizing the worst-case regret associated with arbitrary data.
This simple form of MDL can also be refined to compare and optimize hypothesis classes instead of individual hypotheses, which corresponds to the Bayesian model-class selection problem.
\citet{Grunwald2007} provide an in-depth exposition. 

\tAdd{While there is some similarity between our approach and that of MDL, our theory is driven by a comprehensive treatment of information, and a consistent derivation of prior belief from complexity.
MDL is also consistent with some objectivist Bayesian formulations, such as Jeffrey's priors, however the philosophical motivation is quite different.
Although optimizing MDL encoding length drives at a notion of simplicity, it is not framed within an extended logic to update beliefs from prior or intermediate results, subjective or otherwise. 
Likewise, Minimum Message Length (MML) \citep{Wallace1968,Boulton1975,Wallace1987} is a Bayesian framework that is similar to MDL.
Instead of optimizing hypothesis encodings to minimize worst-case regret, MML minimizes expected code length, which depends on a subjective prior over the attributes a code describes.
We assert that a consistent treatment of both information and prior complexity is critical in the abstract setting of machine learning.}

Further, we highlight the deeper understanding of optimal representations that our theory provides compared to MDL and MML.
Optimization is fundamentally inconsistent with Bayesian probability theory;
inference compels posterior belief from a prior state and the result expresses our rational belief in possible models.
Yet, we recognize that a choice must be made to simplify this process so that problems can be solved on machines with finite resources.
This is why we must distinguish rational choice from rational belief.
Rational choices are informed by rational belief, but also require a utility function.
Building on Bernardo's work \citep{Bernardo1979}, \Cref{cor:model_inference,cor:hyper_inference,cor:properutility,cor:perturbations}
show the variety of circumstances in which information is a proper utility function that serves to guide well-posed optimization for rational choices.
The rational choice becomes the representation we use to approximate posterior belief for future predictions.
While a rational choice could be a single model, as in MDL and MML,
other representations, such as the ensembles in our experiments, have greater utility and provide better prediction uncertainty quantification.

\subsection{Summary and Conclusion}

We proposed Parsimonious Inference, a complete theory of learning based on an information-theoretic formulation of Bayesian inference that quantifies and suppresses a general notion of explanatory complexity.
We showed how our information-theoretic objective allows us to understand the relationship between model complexity and increased agreement between predictions and data labels.

Within the Bayesian perspective, once the prior, the likelihood, and the data are specified, the posterior inexorably follows.
Yet, when we consider the infinite varieties of algorithms that may be developed in machine learning, we find that any universal prior that reserves some degree of plausibility for an arbitrary algorithm becomes uncomputable in practice.
Our framework allows us to resolve the imperative of utility by quantifying the value of a choice,
wherein we only consider a feasible set of prior beliefs and posterior approximations.
By accounting for model complexity from first principles, we can evaluate the utility of such restrictions within a single framework to obtain well-justified predictions within a practical computational budget.

A central aspect of our framework is the distinction between the intrinsic meaning of a potential state of belief and an efficient encoding of that state.
Encoding complexity provides a critical missing component that is needed to measure the complexity of arbitrary inference architectures and naturally associate complexity with plausibility.
Our formulation of generalized length allows us to assign length to a wide variety of codes, beyond binary codes that are typically associated with program length.
We examined some elementary codes to express integers and fractions on the open interval, which can be mapped to a broad class of numbers that may prove useful to represent prior beliefs.

We showed how feasibility-constrained optimizers satisfy quantifiable memorization bounds in comparison to models that may produce better adherence to training data,
but at the cost of increased description length, increased inference information, or information generated by an approximating distribution proposed to generate predictions.
Our experimental results show how our hyperposterior ensembles avoid developing artifacts that artificially hew to seen data within the predictive structure.
Moreover, accounting for multiple explanations by hyperposterior sampling allows us to compute extrapolation uncertainty from first principles as the input domain deviates from past observations.
These experimental results demonstrate how our theory allows us to obtain predictions from extremely small datasets without cross-validation.

Our theory solves critical challenges in understanding how to efficiently learn from data, obtain well-grounded justification for uncertainty in predictions,
and anticipate extrapolation regimes where additional data would prove most beneficial, thus opening a new domain of predictive capabilities. 
This work also provides a principled foundation to address the challenge of feasible learning in the face of high dimensionality.


\section*{Acknowledgements}

We would like to extend our earnest appreciation to Jaideep Ray, Justin Jacobs, and Philip Kegelmeyer for several helpful discussions on this topic.
We also acknowledge and appreciate a conversation with Andrew Charman that clarified our view of the distinction between rational belief and the inherently restrictive nature of a choice.
We sincerely appreciate reviewer feedback that helped us improve this work.

\noindent\textbf{Funding:} This work was funded, in part, by the U.S.~Department of Energy.

Sandia National Laboratories is a multimission laboratory managed and operated by National Technology and Engineering Solutions of Sandia, LLC.,
a wholly owned subsidiary of Honeywell International, Inc., for the U.S. Department of Energy's National Nuclear Security Administration under contract
DE-NA-0003525. This paper describes objective technical results and analysis.
Any subjective views or opinions that might be expressed in the paper do not necessarily represent the views of the U.S. Department of Energy
or the United States Government.


\appendix

\section{Selected Information Corollaries}
\label{sec:corollaries}

We provide proofs of \Cref{thm:info} and the following corollaries in our previous work \citep{Duersch2020}.

\begin{corollary}
\label{cor:chain} \thmTitle{Chain rule of conditional dependence.}
Information associated with joint variables decomposes as
\submitBlank
\begin{align*}
\info{\pR(\vZ_1, \vZ_2)}{\pQ_1(\vZ_1, \vZ_2)}{\pQ_0(\vZ_1, \vZ_2)} =
&\; \info{\pR(\vZ_1)}{\pQ_1(\vZ_1)}{\pQ_0(\vZ_1)} \\
&+ \expect_{\pR(\vZ_1)} \info{\pR(\vZ_2 \mid \vZ_1)}{\pQ_1(\vZ_2 \mid \vZ_1)}{\pQ_0(\vZ_2 \mid \vZ_1)}
\end{align*}
where $\pR(\vZ_1, \vZ_2) = \pR(\vZ_2 \mid \vZ_1)\pR(\vZ_1)$ and $\pQ(\vZ_1, \vZ_2) = \pQ(\vZ_2 \mid \vZ_1)\pQ(\vZ_1)$. 
\end{corollary}

\begin{corollary}
\label{cor:additive} \thmTitle{Additivity over belief sequences.}
Information gained over a sequence of belief updates is additive within the same view.
Given initial belief $\pQ_0(\vZ)$, intermediate states $\pQ_1(\vZ)$ and $\pQ_2(\vZ)$, and the view $\pR(\vZ)$ we have
\submitBlank
\begin{displaymath}
\info{\pR(\vZ)}{\pQ_2(\vZ)}{\pQ_0(\vZ)} = \info{\pR(\vZ)}{\pQ_2(\vZ)}{\pQ_1(\vZ)} + \info{\pR(\vZ)}{\pQ_1(\vZ)}{\pQ_0(\vZ)}.
\end{displaymath}
\end{corollary}

\begin{corollary}
\label{cor:antisymmetry} \thmTitle{Antisymmetry.}
Information from $\pQ_1(\vZ)$ to $\pQ_0(\vZ)$ is the negative of information from $\pQ_0(\vZ)$ to $\pQ_1(\vZ)$
\submitBlank
\begin{displaymath}
\info{\pR(\vZ)}{\pQ_0(\vZ)}{\pQ_1(\vZ)} = -\info{\pR(\vZ)}{\pQ_1(\vZ)}{\pQ_0(\vZ)}.
\end{displaymath}
\end{corollary}

\begin{corollary}
\label{cor:properutility} \thmTitle{Information is a proper utility function.}
Taking the rational view $\pP(\vZ \mid \vX)$ over the latent variable $\vZ$ conditioned upon an experimental outcome $\vX$,
the information $\info{\pP(\vZ \mid \vX)}{\pQ(\vZ)}{\pP(\vZ)}$ from prior belief $\pP(\vZ)$ to reported belief $\pQ(\vZ)$ is a proper utility function.
That is, the unique optimizer recovers rational belief
\submitBlank
\begin{displaymath}
\pQ^*(\vZ) = \argmax_{\pQ(\vZ)}\, \info{\pP(\vZ \mid \vX)}{\pQ(\vZ)}{\pP(\vZ)} = \pP(\vZ \mid \vX).
\end{displaymath}
\end{corollary}

\begin{corollary}
\label{cor:perturbations}
\thmTitle{Proper perturbation response.}
Let $\pQ_1(\vZ)$ be measurably distinct from the view $\pR(\vZ)$ and $\info{\pR(\vZ)}{\pQ_1(\vZ)}{\pQ_0(\vZ)}$ be finite.
Let the perturbation $\pEta(\vZ)$ preserve normalization and drive belief toward $\pR(\vZ)$ on all measurable subsets.
It follows
\submitBlank
\begin{displaymath}
\lim_{\eps\to0} \frac{\partial}{\partial \eps} \info{\pR(\vZ)}{\pQ_1(\vZ) + \eps \pEta(\vZ)}{\pQ_0(\vZ)} > 0.
\end{displaymath}
\end{corollary}


\section{Proofs}
\label{sec:proofs}

\begin{proof}\textbf{of \Cref{cor:length_lowerbound}.}
The primary complication is that each alphabet $\sAlp_{j+1}$ depends on previously realized symbols $(\vS)_{1}^{j}$.
We proceed by induction.
The induction hypothesis regarding a partial sequence $(\vS)_{1}^j$ for $j < n$ is
\submitBlank
\begin{align*}
\expect_{\pP((\vS)_{1}^j)} \left[ \sum_{i=1}^j \log_2\left( |\sAlp_{i}| \right) \right]
\ge \expect_{\pP((\vS)_{1}^j)} \left[ \sum_{i=1}^j \log_2\left( \frac{1}{\pP(\vS_i \mid (\vS)_{1}^{i-1})} \right) \right].
\end{align*}
The base case associated with the first symbol $\vS_1$ easily follows from Jensen's inequality as
\begin{align*}
& \expect_{\pP(\vS_1)} \log_2\left( |\sAlp_1| \right) = \log_2\left( |\sAlp_1| \right) = \log_2\left(\sum_{\vS_1 \in \sAlp_1} 1 \right)\\
& \ge \sum_{\vS_1 \in \sAlp_1} \pP(\vS_1) \log_2\left( \frac{1}{\pP(\vS_1)} \right)
 =  \expect_{\pP(\vS_1)} \log_2\left( \frac{1}{\pP(\vS_1)} \right).
\end{align*}
The induction step is given by applying Jensen's inequality and the induction hypothesis
\submitBlank
\begin{align*}
& \expect_{\pP((\vS)_{1}^{j+1})} \left[ \sum_{i=1}^{j+1} \log_2\left( |\sAlp_i| \right) \right] \\
& = \expect_{\pP((\vS)_{1}^j)} \left[ \sum_{i=1}^{j} \log_2\left( |\sAlp_i| \right) + \expect_{\pP(\vS_{j+1} \mid (\vS)_{1}^j)} \left[ \log_2\left( |\sAlp_{j+1}| \right) \right] \right] \\
& = \expect_{\pP((\vS)_{1}^j)} \left[ \sum_{i=1}^{j} \log_2\left( |\sAlp_i| \right) + \log_2\left( |\sAlp_{j+1}| \right)\right] \\
& \geq \expect_{\pP((\vS)_{1}^j)} \left[ \sum_{i=1}^j \log_2\left( \frac{1}{\pP(\vS_i \mid (\vS)_{1}^{i-1})} \right) + \expect_{\pP(\vS_{j+1} \mid (\vS)_{1}^j)} \left[ \log_2\left( \frac{1}{\pP(\vS_{j+1} \mid (\vS)_{1}^j)} \right) \right] \right] \\
& = \expect_{\pP((\vS)_{1}^{j+1})} \left[ \sum_{i=1}^{j+1} \log_2\left( \frac{1}{\pP(\vS_i \mid (\vS)_{1}^{i-1})} \right) \right].
\end{align*}
The claim follows by noting that $\pP(\vObj) = \pP((\vS)_{1}^{n}) = \prod_{i=1}^n \pP(\vS_i \mid (\vS)_{1}^{i-1})$. 
\end{proof}

\begin{proof}\textbf{of \Cref{cor:probability_from_length}.}
If the encoding is consistent with object probability, and the encoding also maximizes entropy, then
\submitBlank
\begin{align*}
&\pP(\vObj) = \prod_{i=1}^n \pP(\vS_i \mid (\vS)_{1}^{i-1}) = \prod_{i=1}^n \frac{1}{|\sAlp_i|} = 2^{-\fLen(\vObj)}.
\end{align*}
\end{proof}

\begin{proof}\textbf{of \Cref{cor:probability_from_length}, recursive consistency alternative.}
We also obtain the same prior by combining a counting argument with consistent prior belief in length descriptions.
In order to distinguish a complete short sequence from merely a subsequence of a longer description, we need some indication of the complete sequence length.
The following argument holds for descriptions $\vPsi$ that are capable of being partitioned into a component $\vPsi_{\vL}$ that determines the sequence length 
and the unconstrained complement $\vPsi_{\vC}$ so that $\vPsi = (\vPsi_{\vL}, \vPsi_{\vC})$ and $\fLen(\vPsi) = \fLen(\vPsi_{\vL}) + \fLen(\vPsi_{\vC})$.
Once $\vPsi_{\vL}$ is known, we can easily identify $\vPsi_{\vC}$, regardless of its content.
Yet, the same problem arises with knowing when we have a complete length description $\vPsi_{\vL}$, which can be resolved with recursive partitions $\vPsi_{\vL} = (\vPsi_{\vL\vL}, \vPsi_{\vL\vC})$ and so on.
For the description to be finite, this recursion must end implicitly with only one possible outcome $\vPsi_{\vL\ldots\vL\vL}=\emptyset$.

We define $\fLen(\vPsi_{\vC})$ so that the complement allows for $2^{\fLen(\vPsi_{\vC})}$ outcomes.
Binary sequences provide useful intuition for this construction. 
If we believe all descriptions of a given length have the same prior probability, then
\submitBlank
\begin{align*}
1 = \integ{d\vPsi_{\vC}} \pP(\vPsi_{\vC} \mid \fLen(\vPsi)) = 2^{\fLen(\vPsi_{\vC})}  \pP(\vPsi_{\vC} \mid \fLen(\vPsi))
\quad\text{so that}\quad
 \pP(\vPsi_{\vC} \mid \fLen(\vPsi)) = 2^{\fLen(\vPsi) - \fLen(\vPsi_{\vL})}.
\end{align*}
Thus, $\pP(\vPsi) =  2^{\fLen(\vPsi) - \fLen(\vPsi_{\vL})}\pP(\vPsi_{\vL})$.
But if the length description satisfies a consistent prior to determine our belief in the length of the sequence,
we must also have $\pP(\vPsi_{\vL}) =  2^{\fLen(\vPsi_{\vL}) - \fLen(\vPsi_{\vL\vL})}\pP(\vPsi_{\vL\vL})$.
This gives
\submitBlank
\begin{align*}
\pP(\vPsi) = 2^{\fLen(\vPsi)-\fLen(\vPsi_{\vL\vL})} \pP(\vPsi_{\vL\vL}).
\end{align*}
Since the recursion ends with only one outcome, we have $\fLen(\vPsi_{\vL\ldots\vL\vL}) = 0$ and $\pP(\vPsi_{\vL\ldots\vL\vL}) = 1$.
Thus, $\fLen(\vPsi) = \fLen(\vPsi_{\vC}) + \fLen(\vPsi_{\vL\vC}) + \cdots + \fLen(\vPsi_{\vL\ldots\vL\vC})$ and $\pP(\vPsi) = 2^{-\fLen(\vPsi)}$.
\end{proof}

\begin{proof}\textbf{of \Cref{cor:solomonoff}.}
Let $\pQsObj$ be the optimizer.
We express arbitrary infinitesimal belief perturbations in the vicinity of the optimizer as $\pQObj = \pQsObj + \eps \pEta(\vObj)$
where $\eps$ is a scalar differential element and $\pEta(\vObj)$ is an arbitrary perturbation, so long as $\pQsObj > 0$, so that 
\submitBlank
\begin{align*}
\sum_{\vObj \in \sSubObj} \pQObj = \sum_{\vObj \in \sSubObj} \pQsObj = 1
\quad\text{and}\quad
\sum_{\vObj \in \sSubObj} \pEta(\vObj) = 0. 
\end{align*}
The information gained from prior belief to $\pQObj$ is
\submitBlank
\begin{align*}
\KL{\pQObj}{\pP(\vObj)} = \sum_{\vObj \in \sSubObj} \pQObj \log_2\!\left( \frac{\pQObj}{2^{-\fLen(\vPsi(\vObj))}} \right).
\end{align*}
Differentiating with respect to $\eps$, evaluating at $\eps=0$, and applying the variational principle yields
\submitBlank
\begin{align*}
& 0 = \left[ \frac{\partial}{\partial \eps} \info{\pQObj}{\pQObj}{\pP(\vObj)} \right]_{\eps = 0}
= \sum_{\vObj \in \sSubObj} \pEta(\vObj) \left( \log_2\!\left( \frac{\pQsObj}{2^{-\fLen(\vPsi(\vObj))}} \right) + 1 \right).
\end{align*}
As this must hold for arbitrary $\pEta(\vObj)$, the factor in parenthesis must be constant, provided $\pQsObj>0$.
Solving for $\pQsObj$ shows that the stated distribution is the unique critical point.
It remains to show that this distribution achieves the global minimum.
Applying Jensen's inequality to the information gained from any feasible state gives
\submitBlank
\begin{align*}
& \KL{\pQObj}{\pP(\vObj)} = -\sum_{\vObj \in \sSubObj} \pQObj \log_2\!\left( \frac{2^{-\fLen(\vPsi(\vObj))}}{\pQObj} \right) \\
& \ge -\log_2\!\left( \sum_{\vObj \in \sSubObj} 2^{-\fLen(\vPsi(\vObj))} \right)
= \log_2\!\left( \frac{1}{\pP(\sSubObj)} \right)
= \KL{\pQsObj}{\pP(\vObj)}.
\end{align*}
\end{proof}

\begin{proof}\textbf{of \Cref{thm:parsimony_optimization}.}
The chain rule of conditional dependence, \Cref{cor:chain}, allows us to express the information gained as
\submitBlank
\begin{align*}
& \KL{\pRY \pQThCPsi \pQPsi}{ \pP(\vY \mid \vTh) \pP(\vTh \mid \vPsi) \pP(\vPsi)} \\
& = \info{\pRY \pQThCPsi \pQPsi}{\pRY \pQThCPsi \pQPsi}{\pP(\vY \mid \vTh) \pP(\vTh \mid \vPsi) \pP(\vPsi) } \\
& = \info{\pQPsi}{\pQPsi}{\pP(\vPsi) } + \expect_{\pQPsi} \info{\pQThCPsi}{\pQThCPsi}{\pP(\vTh \mid \vPsi)} \\
& \quad + \expect_{\pQThCPsi\pQPsi} \info{\pRY}{\pRY}{\pP(\vY \mid \vTh)}.
\end{align*}
Combining properties of additivity over belief sequences, \Cref{cor:additive}, with antisymmetry, \Cref{cor:antisymmetry}, we express the argument of expectation in the last term as
\submitBlank
\begin{align*}
&  \info{\pRY}{\pRY}{\pP(\vY \mid \vTh)} \\
& = \info{\pRY}{\pRY}{\pP(\vY \mid \vTh_0)} -  \info{\pRY}{\pP(\vY \mid \vTh)}{\pP(\vY \mid \vTh_0)}
\end{align*}
where $\vTh_0$ is any fixed model that serves as a convenient baseline for measuring predictive information with regard to training labels.
Once the data are observed, the term $\info{\pRY}{\pRY}{\pP(\vY \mid \vTh_0)}$ is a constant, independent of choices $\pQThCPsi\pQPsi$.
Thus we have
\submitBlank
\begin{align*}
\omega 
&= \KL{\pRY}{\pP(\vY \mid \vTh_0)} -  \KL{\pRY \pQThCPsi \pQPsi}{ \pP(\vY \mid \vTh) \pP(\vTh \mid \vPsi) \pP(\vPsi)} \\
& = \expect_{\pQThCPsi\pQPsi} \info{\pRY}{\pP(\vY \mid \vTh)}{\pP(\vY \mid \vTh_0)} 
- \expect_{\pQPsi} \KL{\pQThCPsi}{\pP(\vTh \mid \vPsi)} \\
& \quad - \KL{\pQPsi}{\pP(\vPsi)}.
\end{align*}
We apply \Cref{cor:probability_from_length} as a prior over descriptions and unpack the KL divergence to obtain
\submitBlank
\begin{align*}
&-\KL{\pQPsi}{\pP(\vPsi)} =  \int\!\!\,d\vPsi\,\pQPsi \log_2\!\left(\frac{2^{-\fLen(\vPsi)}}{\pQPsi}\right)
= \entropy{\pQPsi} - \expect_{\pQPsi} \fLen(\vPsi). 
\end{align*}
\end{proof}

\begin{proof}\textbf{of \Cref{cor:model_inference}.}
We unpack definitions, combine both terms, and apply Bayes' theorem to obtain
\submitBlank
\begin{align*}
& \expect_{\pQThCPsi} \info{\pRY}{\pP(\vY \mid \vTh)}{\pP(\vY \mid \vTh_0)} - \KL{\pQThCPsi}{\pP(\vTh \mid \vPsi)} \\
& =  \integ{d\vTh} \pQThCPsi \log_2\!\left( \frac{\pP(\vYr \mid \vTh)\pP(\vTh \mid \vPsi)}{\pP(\vYr \mid \vTh_0)\pQThCPsi} \right) \\
& =  \integ{d\vTh} \pQThCPsi \log_2\!\left( \frac{\pP(\vTh \mid \vYr, \vPsi)\pP(\vYr \mid \vPsi)}{\pP(\vYr \mid \vTh_0)\pQThCPsi} \right) \\
& =  \log_2\!\left( \frac{\pP(\vYr \mid \vPsi)}{\pP(\vYr \mid \vTh_0)} \right) - \KL{\pQThCPsi}{\pP(\vTh \mid \vYr, \vPsi)}.
\end{align*}
The second term is the negative Kullback-Leibler divergence, i.e.~nonpositive.
Therefore the objective is maximized when the second term vanishes, which occurs if and only if $\pQsThCPsi = \pP(\vTh \mid \vYr, \vPsi)$.
\end{proof}

\begin{proof}\textbf{of \Cref{cor:hyper_inference}.}
As in \Cref{cor:model_inference}, we unpack definitions, combine both terms, and apply Bayes' theorem to obtain
\submitBlank
\begin{align*}
& \expect_{\pQPsi} \log_2\!\left( \frac{\pP(\vYr \mid \vPsi)}{\pP(\vYr \mid \vTh_0)} \right) - \KL{\pQPsi}{\pP(\vPsi)}
=\sum_{\vPsi}  \pQPsi \log_2\!\left( \frac{\pP(\vYr \mid \vPsi) \pP(\vPsi)}{\pP(\vYr \mid \vTh_0) \pQPsi} \right) \\
&=\sum_{\vPsi}  \pQPsi \log_2\!\left( \frac{\pP(\vPsi \mid \vYr) \pP(\vYr)}{\pP(\vYr \mid \vTh_0) \pQPsi} \right)
=\log_2\!\left( \frac{\pP(\vYr)}{\pP(\vYr \mid \vTh_0)} \right) - \KL{\pQPsi}{\pP(\vPsi \mid \vYr)}.
\end{align*}
The objective is maximized if and only if the second term vanishes, thus $\pQsPsi = \pP(\vPsi \mid \vYr)$.
\end{proof}

\begin{proof}\textbf{of \Cref{cor:memorization}.}
Applying Jensen's inequality to expected prediction information from the optimizer gives 
\submitBlank
\begin{align*}
& \info{\pRY}{\pQsY}{\pP(\vY \mid \vTh_0)} - \vChi[\pQsThPsi] \\
& \geq \expect_{\pQsThPsi} \info{\pRY}{\pP(\vY \mid \vTh)}{\pP(\vY \mid \vTh_0)} - \vChi[\pQsThPsi] \\
& \geq \expect_{\pQThPsi} \info{\pRY}{\pP(\vY \mid \vTh)}{\pP(\vY \mid \vTh_0)} - \vChi[\pQThPsi].
\end{align*}
Since $\info{\pRY}{\pQsY}{\pP(\vY \mid \vTh_0)} = \expect_{\pQThPsi} \info{\pRY}{\pQsY}{\pP(\vY \mid \vTh_0)}$,
we can apply antisymmetry and additivity within the expectations to arrive at the stated result.
\end{proof}


\bibliography{main_arxiv}

\end{document}